\documentclass{article}
\pdfoutput=1
\usepackage[final]{nips_2017}
\usepackage{xspace}
\usepackage{multibib}
\usepackage{graphicx}
\usepackage{subfigure}
\usepackage{wrapfig}
\usepackage{amssymb}
\usepackage{amsmath}
\usepackage{amsthm}
\usepackage{mathrsfs}
\usepackage{array}
\usepackage{multirow}
\usepackage{bbm}
\usepackage{color}
\usepackage{algorithm} %format of the algorithm
\usepackage{algorithmic} %format of the algorithm
\usepackage{epstopdf}
\usepackage{footmisc}
\newcommand{\argmin}{\mathop{ \arg\!\min}}
\newcommand{\arginf}{\mathop{ \arg\!\inf}}
\renewcommand{\vec}[1]{\mathbf{#1}}

\usepackage{url}
\usepackage{epsfig}
\usepackage{subfigure}
\usepackage{graphicx}
\usepackage{amsmath,amsfonts,amsthm,mathrsfs}
\usepackage{multirow}
\usepackage{paralist}
\usepackage{color}

\def\lt{\left}
\def\rt{\right}

\def\x{\vec{x}}

\def\w{\vec{w}}
\def\R{{\cal R}}
\def\L{{\cal L}}

\def\bu{\mathbf{u}}

\def\bv{\mathbf{v}}

\def\bx{\mathbf{x}}
\def\bz{\mathbf{z}}

\def\bp{\mathbf{p}}
\def\bone{\mathbf{1}}
\def\bzero{\mathbf{0}}
\def\sgn{\hbox{sign}}

\def\eg{\emph{e.g.}} 
\def\ie{\emph{i.e.}}

\def\st{\text{s.t.}}

\def\N{\mathbb{N}}
\def\R{\mathbb{R}}

\def\X{\mathcal{X}}
\def\Y{\mathcal{Y}}
\def\Z{\mathcal{Z}}

\def\C{\mathcal{C}}

\def\L{\mathcal{L}}

\def\H{\mathcal{H}}
\def\A{{\mathcal{A}}}

\def\EX{{\mathbb{E}}}

\def\bx{{\bf x}}
\def\beg{\begin}

\def\gl{\lambda}
\def\R{\mathbb{R}}

\def\cR{{\mathcal{R}}}
\def\cE{{\mathcal{E}}}

\def\beg{\begin}

\def\mbI{\mathbb{I}}

\def\U{\mathcal{U}}

\def\ie{{\em i.e.}}
\def\eg{{\em e.g.}}

\newtheorem{theorem}{Theorem}
\newtheorem{corollary}{Corollary}
\newtheorem{lemma}{Lemma}

\newtheorem{proposition}{Proposition}
\newtheorem{definition}{Definition}
\newtheorem{remark}{Remark}
\newtheorem{example}{Example}

\def\dsum{\displaystyle\sum }

\def\begeqn{\begin{equation}}
\def\endeqn{\end{equation}}
\def\begth{\begin{theorem}}
\def\endth{\end{theorem}}
\def\begprop{\begin{proposition}}
\def\endprop{\end{proposition}}
\def\begcor{\begin{corollary}}
\def\endcor{\end{corollary}}
\def\begdef{\begin{definition}}
\def\enddef{\end{definition}}
\def\beglemm{\begin{lemma}}
\def\endlemm{\end{lemma}}
\def\begexm{\begin{example}}
\def\endexm{\end{example}}
\def\begrem{\begin{remark}}
\def\endrem{\end{remark}}
\def\beg{\begin}

\def\ga{\alpha}

\def\gb{\beta}

\def\gep{\varepsilon}

\def\gk{\kappa}
\def\gl{\lambda}

\def\gs{\sigma}

\def\gO{\Omega}

\def\bz{{\bf z}}
\def\bx{{\bf x}}

\def\N{\mathbb{N}}
\def\R{\mathbb{R}}

\def\X{\mathcal{X}}
\def\Y{\mathcal{Y}}
\def\Z{\mathcal{Z}}

\def\C{\mathcal{C}}

\def\L{\mathcal{L}}

\def\H{\mathcal{H}}
\def\A{{\mathcal{A}}}

\def\EX{{\mathbb{E}}}

\def\bx{{\bf x}}
\def\beg{\begin}

\def\gl{\lambda}
\def\R{\mathbb{R}}

\def\cR{{\mathcal{R}}}
\def\cE{{\mathcal{E}}}

\def\beg{\begin}

\def\mbI{\mathbb{I}}

\def\U{\mathcal{U}}

\def\ie{{\em i.e.}}
\def\eg{{\em e.g.}}

\def\dmax{\displaystyle\max}
\def\dsum{\displaystyle\sum}
\def\dsup{\displaystyle\sup}
\def\dprod{\displaystyle\prod }

\newcommand\numberthis{\addtocounter{equation}{1}\tag{\theequation}}
\allowdisplaybreaks[4]

\newcites{apx}{Reference}

\def\Lag{\L_\text{avg}}
\def\Lmm{\L_\text{max}}
\def\Lt#1{\L_{\text{top-}#1}}
\def\Lat#1{\L_{\text{avt-}#1}}

\def\matk{MAT$_k$\xspace}
\def\atk{AT$_k$\xspace}

\def\Apx{Appendix}

\title{Learning with Average Top-k Loss}
\author{Yanbo Fan$^{3,4,1}$ 
	, Siwei Lyu$^{1}$\thanks{Corresponding author.}
	, Yiming Ying$^{2}$
	, Bao-Gang Hu$^{3,4}$\\
	{$^{1}$Department of Computer Science, University at Albany, SUNY} \\
	{$^{2}$Department of Mathematics and Statistics, University at Albany, SUNY} \\
	{$^{3}$National Laboratory of Pattern Recognition, CASIA}\\
	{$^{4}$University of Chinese Academy of Sciences (UCAS)}\\
	{\{yanbo.fan,hubg\}@nlpr.ia.ac.cn, slyu@albany.edu, yying@albany.edu}
}

\begin{document}
	\maketitle

	%%%%%%%%%%%%%%%%%%%%%%%  abstract     %%%%%%%%%%%%%%%%%%%%%%%%%%%%
	% !TEX root =  ../draft.tex
\begin{abstract}
In this work, we introduce the {\em average top-$k$} (\atk) loss as a new aggregate loss for supervised learning, which is the average over the $k$ largest individual losses over a training dataset. We show that the \atk loss is a natural generalization of the two widely used aggregate losses, namely the average loss and the maximum loss, but can combine their advantages and mitigate their drawbacks to better adapt to different data distributions. Furthermore, it remains a convex function over all individual losses, which can lead to convex optimization problems that can be solved effectively with conventional gradient-based methods. We provide an intuitive interpretation of the \atk loss based on its equivalent effect on the continuous individual loss functions, suggesting that it can reduce the penalty on correctly classified data.  We further give a learning theory analysis of \matk learning on the classification calibration of the \atk loss and the error bounds of \atk-SVM. We demonstrate the applicability of minimum average top-$k$ learning for binary classification and regression using synthetic and real datasets.
\end{abstract}
	
	%%%%%%%%%%%%%%%%%%%%%%%  introduction %%%%%%%%%%%%%%%%%%%%%%%%%%%%
	% !TEX root =  ../draft.tex
%%%%%%%%%%%%%%%%%%%%%%%
\section{Introduction}

%* general parametric machine learning is to minimize aggregate of individual loss 
%* individual loss vs. aggregate loss operator
%* different types of aggregate loss operators
%    * average
%    * minimax
%* we introduce a more general aggregate loss operator 
%* need mention the motivation of using top-K here, using the toy example
%* not to replace average or minimax, provide alternative
%* this is general, can work with most individual loss and parametric function, though convex has more theoretical guarantee
%* fig 1. panel 1 (dataset & separation) panel 2 (loss function) panel 3 (training error)
%* contribution
%    * we propose a top k average operator, provide interpretations (?) and prove some of its properties
%    * we provide a general efficient algorithm to solve it
%    * we set up connections of top k average examples with existing algorithms (nu-SVM) and provide new interpretations
%    * we show top k average operator is suitable to some practical datasets

Supervised learning concerns the inference of a function $f: \X \mapsto \Y$ that predicts a target $y \in \Y$ from data/features $\x \in \X$ using a set of labeled training examples $\{(\x_i,y_i)\}_{i=1}^n$. This is typically achieved by seeking a function $f$ that minimizes an {\em aggregate loss} formed from {\em individual losses} evaluated over all training samples. 

To be more specific, the individual loss for a sample $(\x,y)$ is given by $\ell(f(\x),y)$, in which $\ell$ is a nonnegative bivariate function that evaluates the quality of the prediction made by function $f$. For example, for binary classification (\ie, $y_i \in \{\pm 1\}$), commonly used forms for individual loss include the $0$-$1$ loss, $\mbI_{yf(\x)\le 0}$, which is $1$ when $y$ and $f(\x)$ have different sign and $0$ otherwise, the hinge loss, $\max(0,1-yf(\x))$, and the logistic loss, $\log_2(1+\exp(-yf(\x)))$, all of which can be further simplified as the so-called {\em margin loss}, \ie, $\ell(y,f(\x)) = \ell(yf(\x))$.  For regression, squared difference $(y-f(\x))^2$ and absolute difference $|y-f(\x)|$ are two most popular forms for individual loss, which can be simplified as $\ell(y,f(\x)) = \ell(|y-f(\x)|)$. Usually the individual loss is chosen to be a convex function of its input, but recent works also propose various types of non-convex individual losses (\eg, \cite{he2011maximum,masnadi2009design, wu2007robust,yang2010relaxed}).

%For multi-class classification (\ie, $\Y =\N_m:=\{1,\cdots,m\}$), we can use the one-versus-all strategy individual loss in the form of \red{can we ignore this??? hard to write the loss function}. 

The supervised learning problem is then formulated as $\min_f \lt\{\L(L_\bz(f)) + \Omega(f)\rt\}$, where $\L(L_\bz(f))$ is the aggregate loss accumulates all individual losses over training samples, \ie, $L_\bz(f) = \{\ell_i(f)\}_{i=1}^n$, with $\ell_i(f)$ being the shorthand notation for $\ell(f(\x_i),y_i)$, and $\Omega(f)$ is the regularizer on $f$.  However, in contrast to the plethora of the types of individual losses, there are only a few choices when we consider the aggregate loss:
~\vspace{-.75em}
\begin{itemize} \itemsep -0.3em
\item the {\em average loss}: $\Lag(L_\bz(f)) = {1 \over n} \sum_{i=1}^n \ell_i(f)$, \ie, the mean of all individual losses;
\item the {\em maximum loss}: $\Lmm(L_\bz(f)) = \max_{1\le k\le n} \ell_i(f)$, \ie, the largest individual loss;
\item the {\em top-$k$ loss} \cite{shalev2016minimizing}: $\Lt{k}(L_\bz(f)) = \ell_{[k]}(f)$\footnote{We define the {\em top-$k$} element of a set $S = \{s_1,\cdots,s_n\}$ as $s_{[k]}$, such that $s_{[1]} \geq s_{[2]} \geq \cdots \geq s_{[n]}$.} for $1 \le k \le n$, \ie, the $k$-th largest (top-$k$) individual loss.
~\vspace{-.5em}
\end{itemize}
The average loss is unarguably the most widely used aggregate loss, as it is a unbiased approximation to the expected risk and leads to the {\em empirical risk minimization} in learning theory \cite{bartlett2006convexity,de2005model,steinwart2003optimal,vapnik,wu2006learning}.  Further, minimizing the average loss affords simple and efficient stochastic gradient descent algorithms \cite{bousquet2008tradeoffs,srebro2010stochastic}. %\ie, if we assume the training data are i.i.d. samples of a joint but unknown distribution $\rho(\x,y)$ on $\X \times \Y$, then $\Lag(L_\bz(f)) \xrightarrow[n\rightarrow \infty]{}\EX_{\rho(\x,y)}[\ell(f(\x),y)]$. 
On the other hand, the work in \cite{shalev2016minimizing} shows that constructing learning objective based on the maximum loss may lead to improved performance for data with separate typical and rare sub-populations. The top-k loss \cite{shalev2016minimizing} generalizes the maximum loss, as $\Lmm(L_\bz(f)) = \Lt{1}(L_\bz(f))$, and can alleviate the sensitivity to outliers of the latter. However, unlike the average loss or the maximum loss, the top-k loss in general does not lead to a convex learning objective, as it is not convex of all the individual losses $L_\bz(f)$. 

In this work, we propose a new type of aggregate loss that we term as the {\em average top-k} (\atk) loss, which is the average of the largest $k$ individual losses, that is defined as:
~\vspace{-.5em} \begin{equation} \textstyle
\Lat{k}(L_\bz(f)) = {1 \over k} \sum_{i=1}^k \ell_{[i]}(f).
\label{eq:0}
\end{equation}
We refer to learning objectives based on minimizing the \atk loss as {\em \matk learning}. 

\begin{figure}[t]
%\begin{center}
\begin{tabular}{c@{\hspace{0em}}c@{\hspace{0.5em}}|@{\hspace{0.5em}}c@{\hspace{0em}}c}
%multi-modal situation
\includegraphics[width=.24\textwidth]{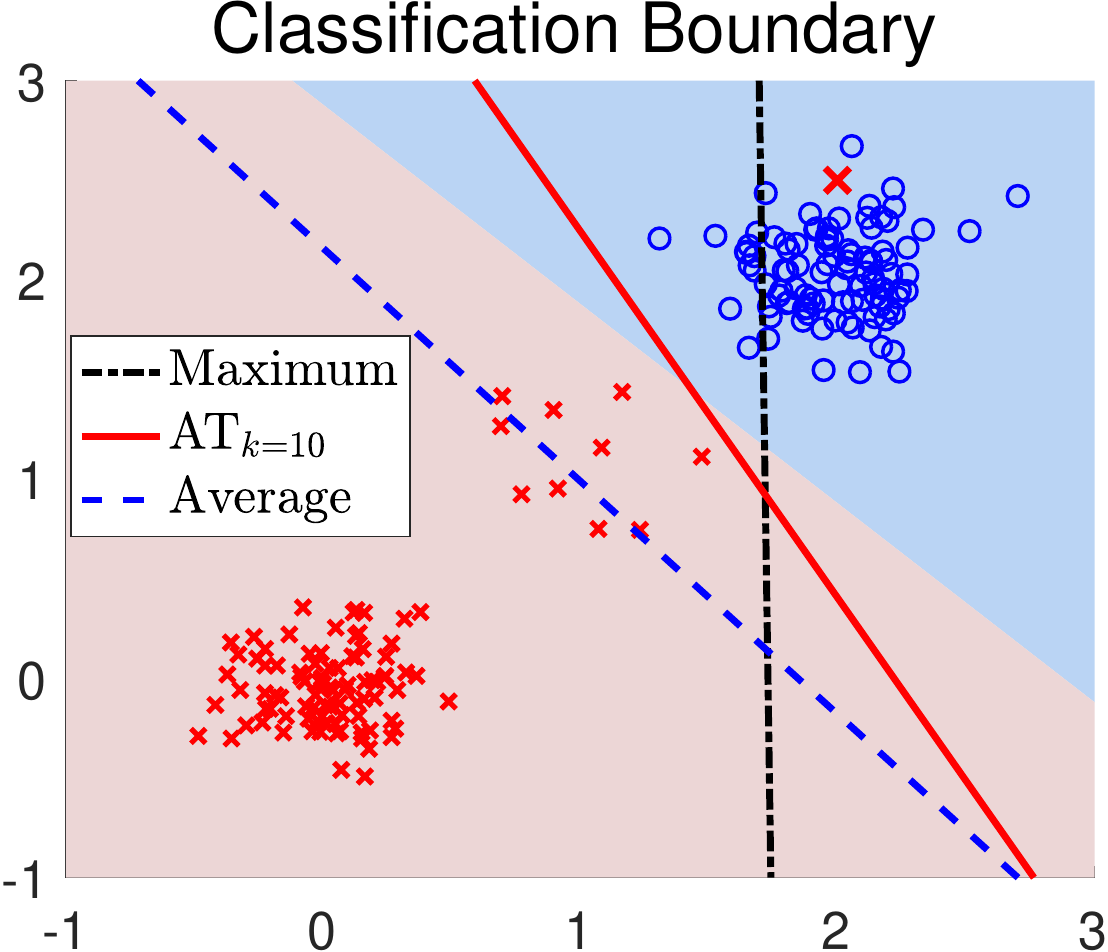} &
\includegraphics[width=.24\textwidth]{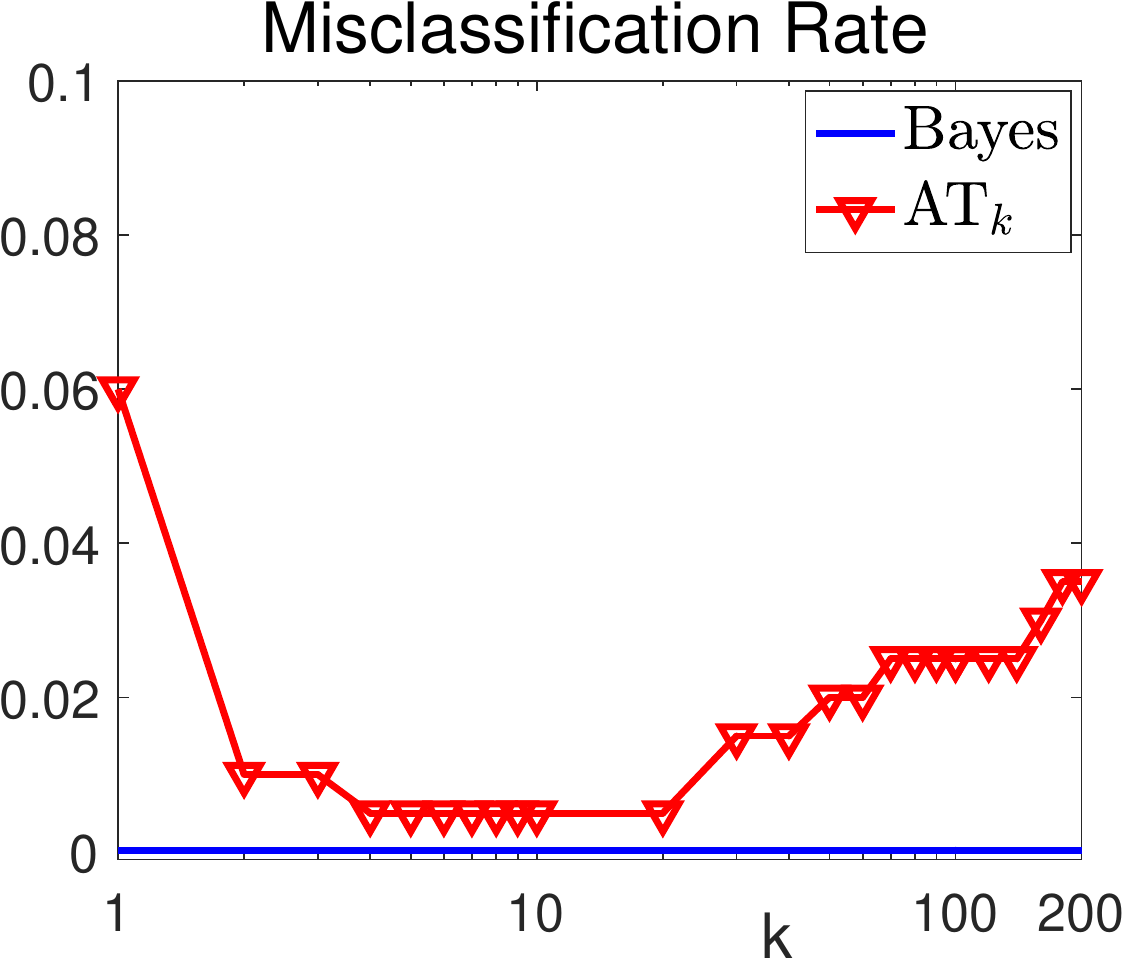}  &
\includegraphics[width=.24\textwidth]{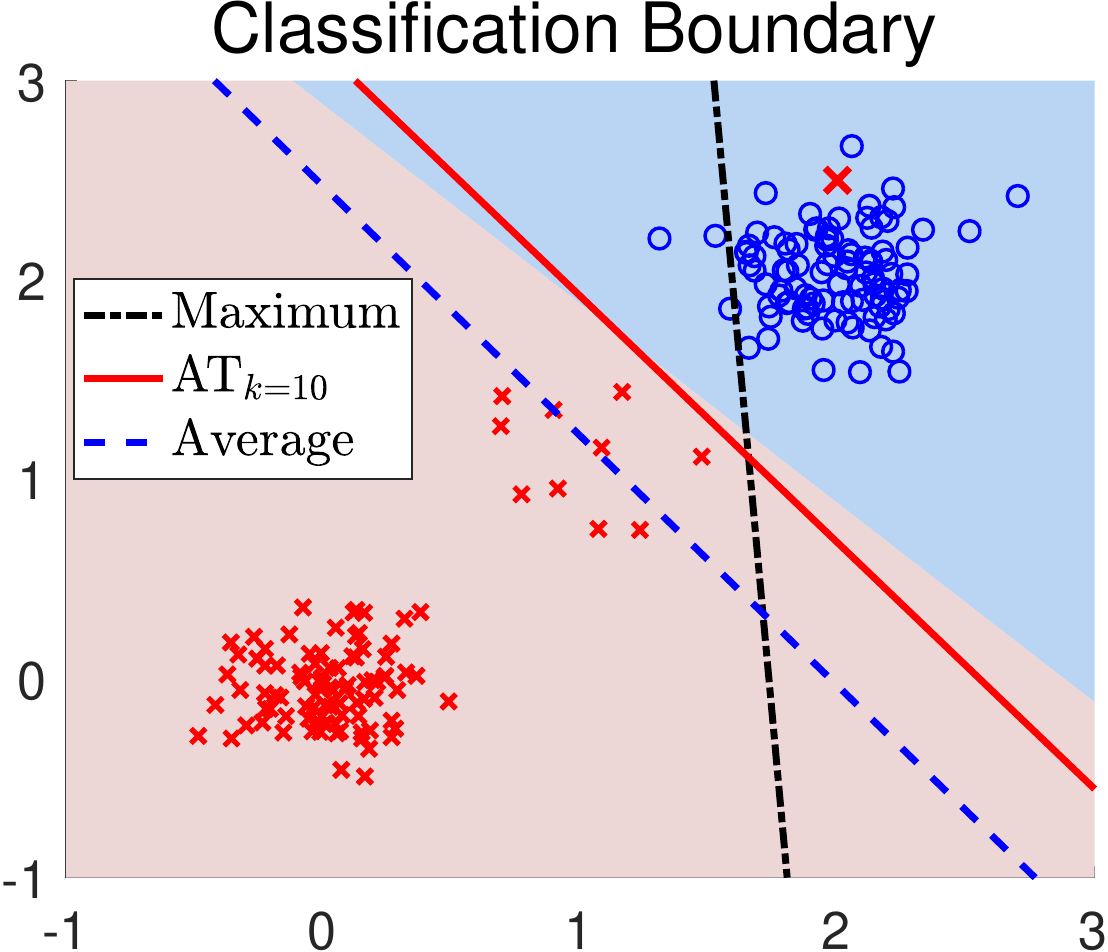} &
\includegraphics[width=.24\textwidth]{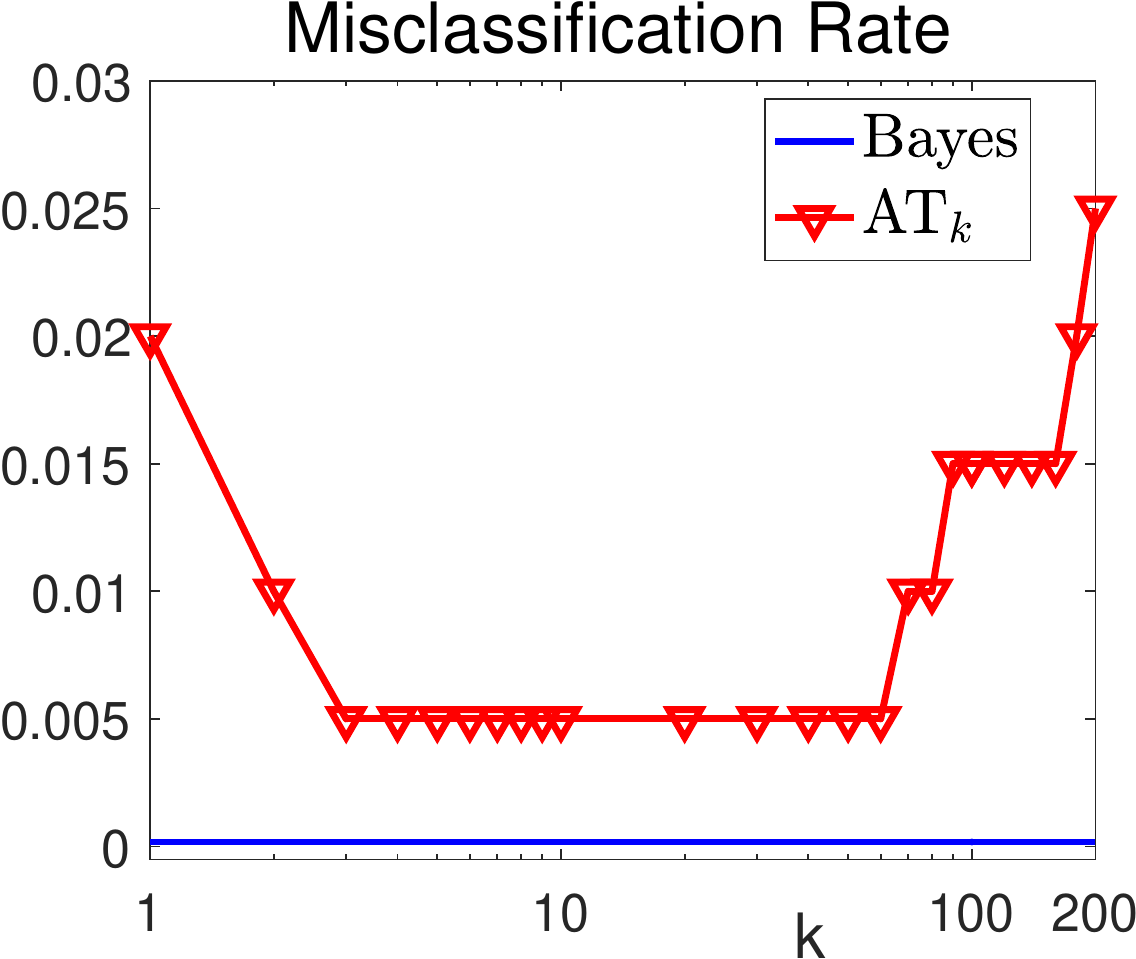} \\
%imbalance situation
\includegraphics[width=.24\textwidth]{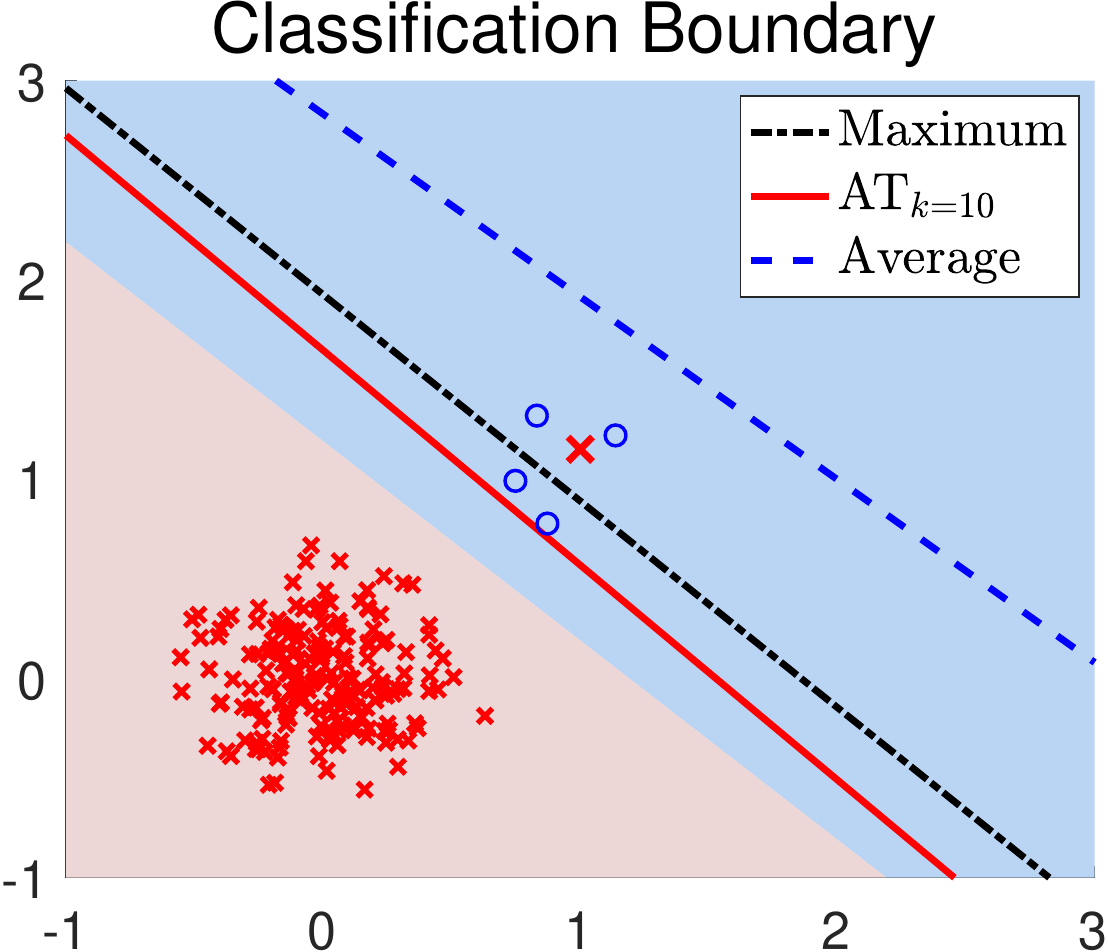} &
\includegraphics[width=.24\textwidth]{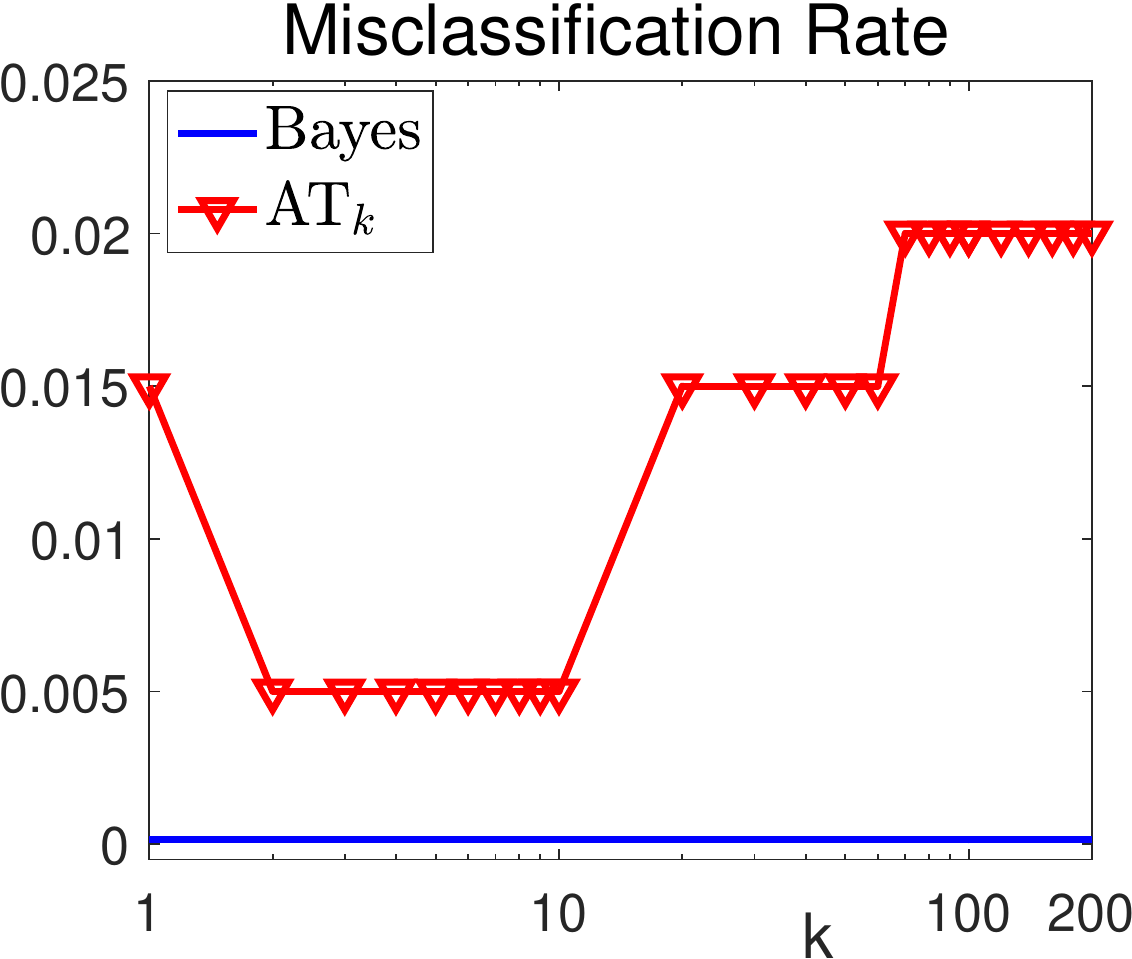}  &
\includegraphics[width=.24\textwidth]{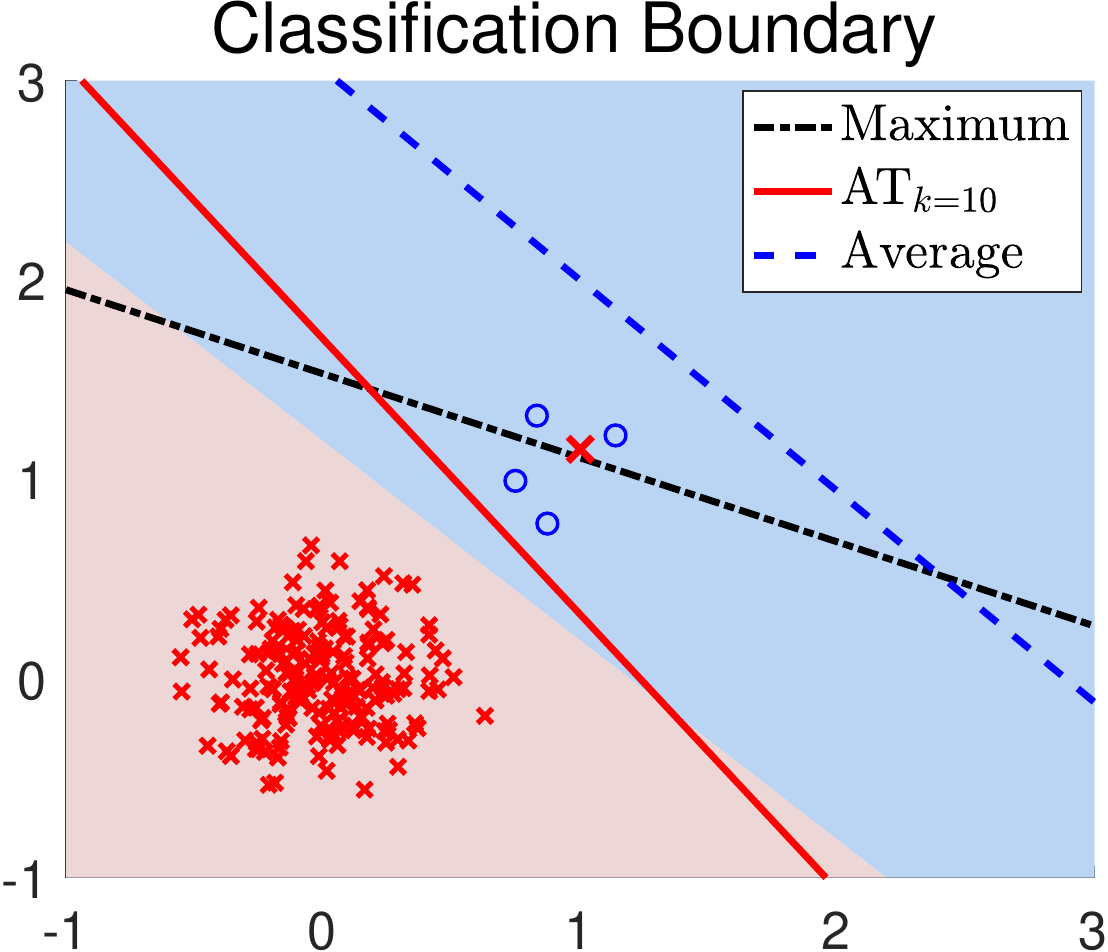} &
\includegraphics[width=.24\textwidth]{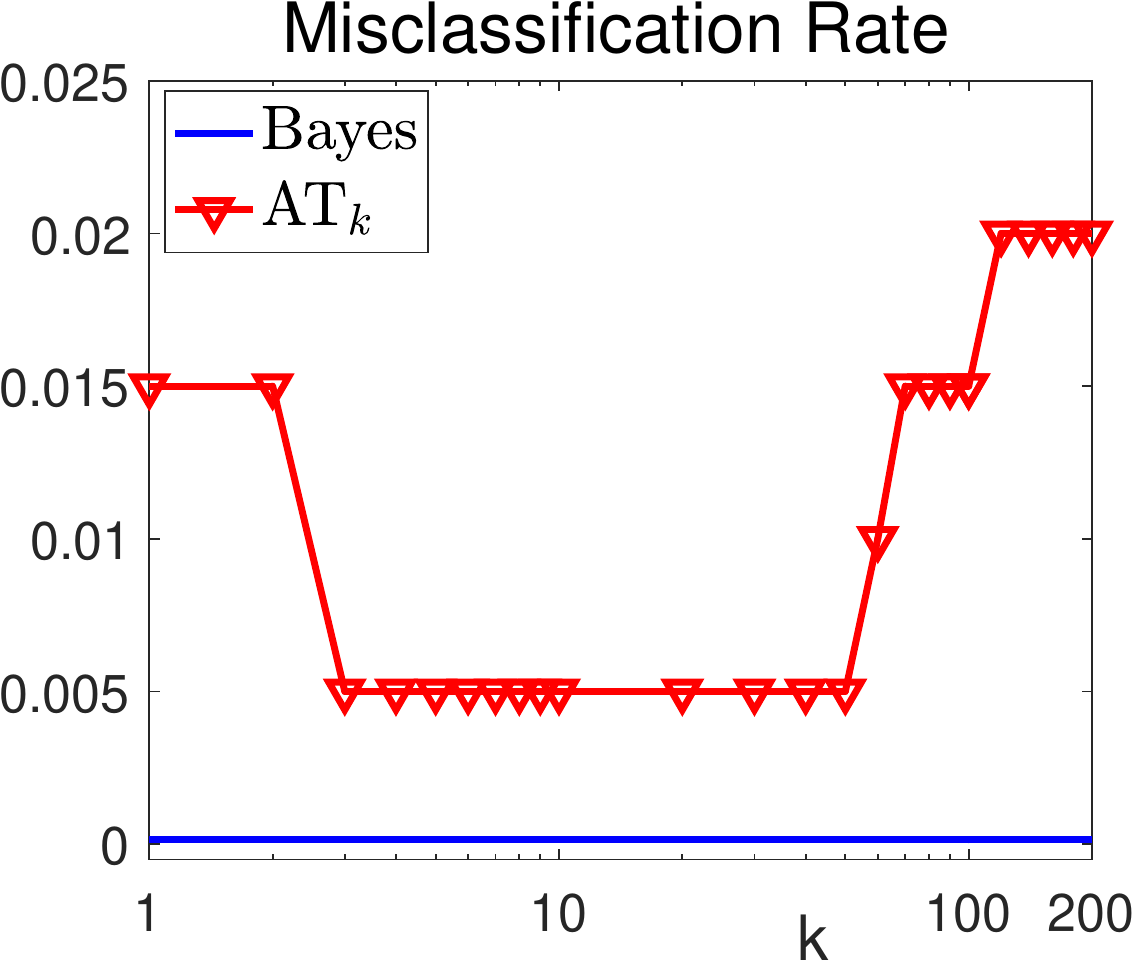} \\
\end{tabular}
%\end{center}
~\vspace{-1em}
\caption{
	\em \small Comparison of different aggregate losses on 2D synthetic datasets with $n=200$ samples for binary classification on a balanced but multi-modal dataset and with outliers ({\bf top}) and an imbalanced dataset with outliers ({\bf bottom}) with logistic loss ({\bf left}) and hinge loss ({\bf right}). 
	Outliers in data are shown as an enlarged $\times$ and the optimal Bayes classifications are shown as shaded areas. The figures in the second and fourth columns show the misclassification rate of \atk vs. $k$ for each case.
}
\label{fig:1}
~\vspace{-2.5em}
\end{figure}
The \atk loss generalizes the average loss ($k=n$) and the maximum loss ($k=1$), yet it is less susceptible to their corresponding drawbacks, \ie, it is less sensitive to outliers than the maximum loss and can adapt to imbalanced and/or multi-modal data distributions better than the average loss. This is illustrated with two toy examples of synthesized 2D data for binary classification in Fig.\ref{fig:1} (see \Apx\: for a complete illustration). As these plots show, the linear classifier obtained with the maximum loss is not optimal due to the existence of outliers while the linear classifier corresponding to the average loss has to accommodate the requirement to minimize individual losses across all training data, and sacrifices smaller sub-clusters of data (\eg, the rare population of $+$ class in the top row and the smaller dataset of $-$ class in the bottom row).
%due to the non-zero penalty to correctly classified samples of individual logistic loss. 
In contrast, using \atk loss with $k=10$ {can better protect such smaller sub-clusters and} leads to linear classifiers closer to the optimal Bayesian linear classifier. This is also corroborated by the plots of corresponding misclassification rate of \atk vs. $k$ value in Fig.\ref{fig:1}, which show that minimum misclassification rates occur at $k$ value other than $1$ (maximum loss) or $n$ (average loss). 
%The \atk loss generalizes the average loss ($k=n$) and the maximum loss ($k=1$), yet it is less susceptible to their drawbacks, \ie, it is less sensitive to outliers and can adapt to imbalanced or multiple modal data distributions. This is illustrated with two toy examples with synthesized 2D data for binary classification in Fig.\ref{fig:1}. The top row shows the linear classifiers learned with different aggregate losses, together with the optimal Bayes classifier (as shaded areas).  The linear classifier obtained with the maximum loss is not optimal due to the existence of outliers while the linear classifier corresponding to the average loss has to accommodate the requirement to minimize individual losses across the board, and sacrifices the smaller clusters of data. In contrast, using \atk loss leads to classifiers much closer to the optimal Bayesian classifier. This is also corroborated by the plots of corresponding classification accuracy vs. $k$ value shown in the bottom row. 

The \atk loss is a tight upper-bound of the top-$k$ loss, as $\Lat{k}(L_\bz(f)) \ge \Lt{k}(L_\bz(f))$ with equality holds when {$k=1$ or} $\ell_i(f) = $ constant, and it is a convex function of the individual losses (see Section \ref{sec:interpretation}). Indeed, we can express $\ell_{[k]}(f)$ as the difference of two convex functions $k \Lat{k}(L_\bz(f)) - (k-1)\Lat{(k-1)}(L_\bz(f))$, which shows that in general $\Lt{k}(L_\bz(f))$ is not convex with regards to the individual losses.

In sequel, we will provide a detailed analysis of the \atk loss and \matk learning. First, we establish a reformulation of the \atk loss as the minimum of the average of the individual losses over all training examples transformed by a hinge function. This reformulation leads to a simple and effective stochastic gradient-based algorithm for \matk learning, and interprets the effect of the \atk loss as shifting down and truncating at zero the individual loss to reduce the undesirable penalty on correctly classified data. When combined with the hinge function as individual loss, the \atk aggregate loss leads to a new variant of SVM algorithm that we term as \atk SVM, which generalizes the C-SVM and the $\nu$-SVM algorithms \cite{scholkopf2000new}. We further study learning theory of \matk learning, focusing on the classification calibration of the \atk loss function and {error bounds} of the \atk SVM algorithm.  This provides a theoretical lower-bound for $k$ for reliable classification performance. We demonstrate the applicability of minimum average top-$k$ learning for binary classification and regression using synthetic and real datasets.

The main contributions of this work can be summarized as follows.
~\vspace{-.75em}
\begin{itemize} \itemsep -0em
\item We introduce the \atk loss for supervised learning, which can balance the pros and cons of the average and maximum losses, and allows the learning algorithm to better adapt to imbalanced and multi-modal data distributions. 

\item We provide algorithm and interpretation of the \atk loss, suggesting that most existing learning algorithms can take advantage of it without significant increase in computation.

\item We further study the theoretical aspects of \atk loss on classification calibration and error bounds of minimum average top-$k$ learning for \atk-SVM.

\item We perform extensive experiments to validate the effectiveness of the \matk learning.
\end{itemize}

	%%%%%%%   Formulation & interpretation & AT$_k$-SVM %%%%%%%%%%%%%%
	% !TEX root =  ../draft.tex

\section{Formulation and Interpretation}
\label{sec:interpretation}
%* reformulation as hinge loss
%    * interpretation as modulating individual loss (use Fig.1 (b) as example)
%        * most loss are continuous convex upper bounds to the sharp 0-1 loss. 
%        * they penalize wrong classification but also correct classification
%        * effect of top-K is to shift the individual loss function w.r.t each data
%            * reducing the overall penalty, 
%                * so correct data get less penalty, which also becomes zero after certain range
%                * does not hurt overall performance as wrong data still get penalty, albeit reduced
%* general algorithm
%    * coordinate descent between lambda and w
%        * lambda step is top-K
%        * w step uses sub gradient descent
%
%* atk-svm
%

The original \atk loss, though intuitive, is not convenient to work with because of the sorting procedure involved. This also obscures its connection with the statistical view of supervised learning as minimizing the expectation of individual loss with regards to the underlying data distribution. Yet, it affords an equivalent form, which is based on the following result.
\begin{lemma}[Lemma 1, \cite{Ogryczak:2003dl}] 
$\sum_{i=1}^{k} x_{[i]}$ is a convex function of $(x_1,\cdots,x_n)$.  Furthermore, for $x_i \ge 0$ and $i = 1,\cdots, n$, we have $\sum_{i=1}^{k} x_{[i]} = \min_{\lambda \ge 0} \left\{ k\lambda + \sum_{i=1}^{n}\left[x_i - \lambda \right]_{+}  \right\}$, where $\left[a\right]_{+} = \max\{0,a\}$ is the hinge function.  
\label{lem:1}
\end{lemma}
For completeness, we include a proof of Lemma \ref{lem:1} in \Apx. Using Lemma \ref{lem:1}, we can reformulate the \atk loss \eqref{eq:0} as
~\vspace{-.5em}\begeqn
\Lat{k}(L_\bz(f)) = {1 \over k} \sum_{i=1}^k \ell_{[i]}(f) \propto \min_{\lambda \ge 0} \lt\{{1\over n}\sum_{i=1}^{n}\left[\ell_i(f) - \lambda \right]_{+} + {k \over n}\lambda \rt\}. \label{eq:atk}
\endeqn 
In other words, the \atk loss is equivalent to minimum of the average of individual losses that are shifted and truncated by the hinge function controlled by $\lambda$. This sheds more lights on the \atk loss, which is particularly easy to illustrate in the context of binary classification using the margin losses, $\ell(f(\x),y) = \ell(yf(\x))$. 
\begin{wrapfigure}{R}{.35\textwidth}
\centering
\includegraphics[width=.35\textwidth]{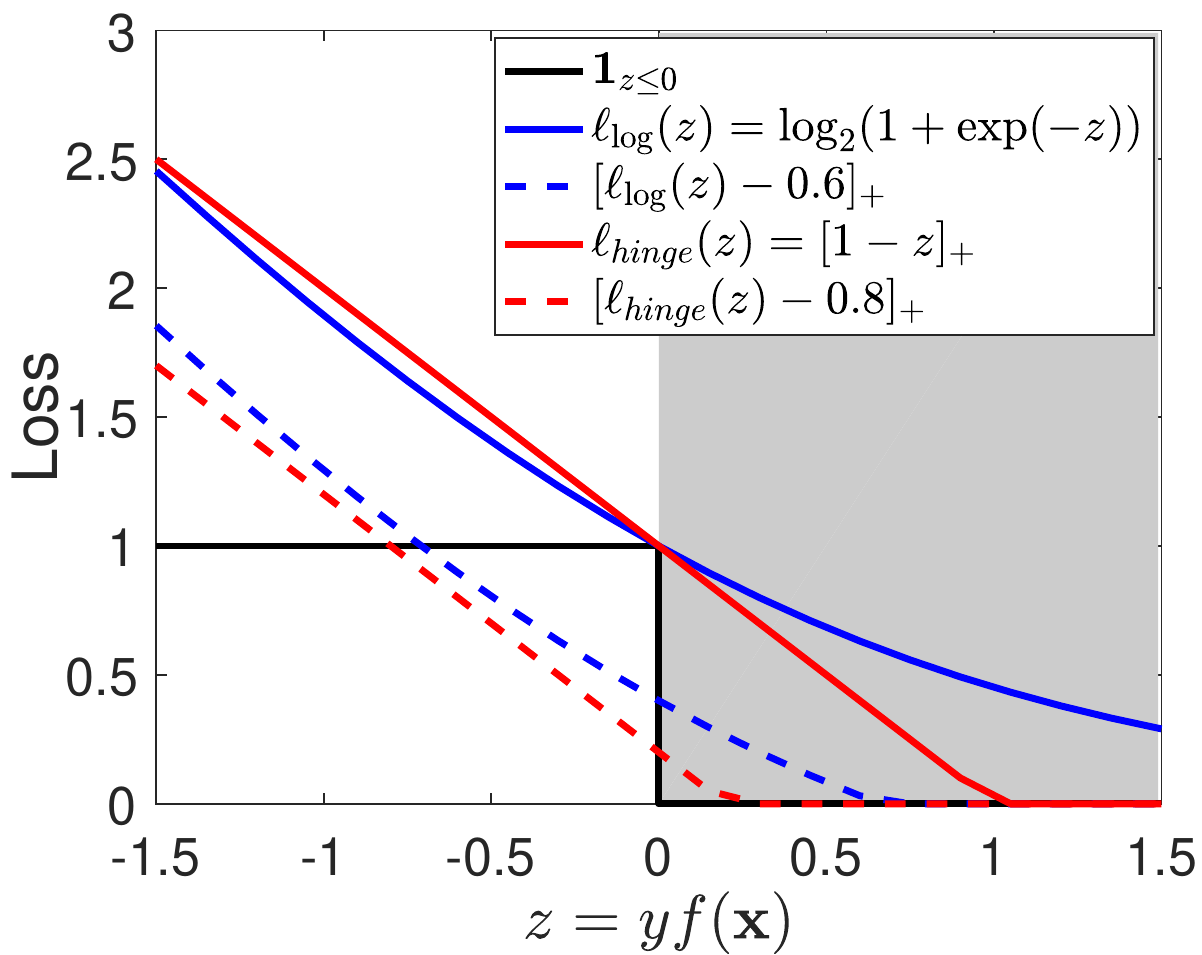}
~\vspace{-2em}
\caption{\em \small The \atk loss interpreted at the individual loss level. Shaded area corresponds to data/target with correct classification.}
\label{fig:2}
~\vspace{-2.5em}
\end{wrapfigure}

In binary classification, the ``gold standard'' of individual loss is the $0$-$1$ loss $\mbI_{yf(\x)\le 0}$, which exerts a constant penalty $1$ to examples that are misclassified by $f$ and no penalty to correctly classified examples. However, the $0$-$1$ loss is difficult to work as it is neither continuous nor convex. In practice, it is usually replaced by a surrogate convex loss. Such convex surrogates afford efficient algorithms, but as continuous and convex upper-bounds of the $0$-$1$ loss, they typically also penalize correctly classified examples, \ie, for $y$ and $\x$ that satisfy $yf(\x) >0$, $\ell(yf(\x)) > 0$, whereas $\mbI_{yf(\x)\le 0} = 0$ (Fig.\ref{fig:2}). This implies that when the average of individual losses across all training examples is minimized, correctly classified examples by $f$ that are ``too close'' to the classification boundary may be sacrificed to accommodate reducing the average loss, as is shown in Fig.\ref{fig:1}.

In contrast, after the individual loss is combined with the hinge function, \ie, $\left[\ell(yf(\x)) - \lambda \right]_{+}$ with $\lambda > 0$, it has the effect of ``shifting down'' the original individual loss function and truncating it at zero, see Fig.\ref{fig:2}. The transformation of the individual loss reduces penalties of all examples, and in particular benefits correctly classified data. In particular, if such examples are ``far enough'' from the decision boundary, like in the $0$-$1$ loss, their penalty becomes zero. This alleviates the likelihood of misclassification on those rare sub-populations of data that are close to the decision boundary. 

\smallskip

\noindent{\bf Algorithm}: The reformulation of the \atk loss in Eq.\eqref{eq:atk} also facilitates development of optimization algorithms for the minimum \atk learning.  As practical supervised learning problems usually use a parametric form of $f$, as $f(\x;\w)$, where $\w$ is the parameter, the corresponding minimum \atk objective becomes
\begin{equation}
\min_{\w, \lambda \ge 0} \lt\{ {1\over n}\sum_{i=1}^{n}\left[\ell(f(\x_i;\w),y_i) - \lambda \right]_{+}  + {k \over n} \lambda + \Omega(\w)\rt\},
\label{eq:1}
\end{equation}
%where $C>0$ is a trade-off parameter. 
 It is not hard to see that if $\ell(f(\x;\w),y)$ is convex with respect to $\w$,  the objective function of in Eq.\eqref{eq:1} is a convex function for $\w$ and $\lambda$ jointly. This leads to an immediate stochastic (projected) gradient descent \cite{bousquet2008tradeoffs,srebro2010stochastic} for solving \eqref{eq:1}. For instance, with $\Omega(\w) = {1 \over 2C} \|\w\|^2$, where $C > 0$ is a regularization factor, at the $t$-th iteration, the corresponding \matk objective can be minimized by first randomly sampling $(\x_{i_t}, y_{i_t})$ from the training set and then updating the parameters as
\begin{equation}
\begin{array}{ll}
\w^{(t+1)} & \leftarrow \w^{(t)} - \eta_t \lt( \partial_{\w}\ell(f(\x_{i_t};\w^{(t)}),y_{i_t}) \cdot \mbI_{[\ell(f(\x_{i_t};\w^{(t)}),y_{i_t}) > \lambda^{(t)}]}  + {\w^{(t)}\over C}\rt) \\
\gl^{(t+1)} & \leftarrow \lt[ \gl^{(t)} - \eta_t \lt( {k \over n}  -  \mbI_{[\ell(f(\x_{i_t};\w^{(t)},y_{i_t}) > \lambda^{(t)}]}\rt)\rt]_+
\end{array}
\label{eq:2}
\end{equation}
where $\partial_{\w}\ell(f(\x;\w),y)$ denotes the sub-gradient with respect to $\w,$ and $\eta_t \sim {1\over \sqrt{t}}$ is the step size. 

\smallskip

\noindent{\bf \atk-SVM}: As a general aggregate loss, the \atk loss can be combined with any functional form for individual losses. In the case of binary classification, the \atk loss combined with the individual hinge loss for a prediction function $f$ from a reproducing kernel Hilbert space (RKHS) \cite{scholkopf2001learning} leads to the \atk-SVM model. Specifically, we consider function $f$ as a member of RKHS $\H_K$ with norm $\|\cdot\|_K$, which is induced from a reproducing kernel $K: \X \times \X \to \R$. Using the individual hinge loss, $[1 - y_i f(\x_i)]_+$,  the corresponding \matk learning objective in RKHS becomes  
\begin{equation}
	\min_{f \in \H_K,\lambda \ge 0}  {1 \over n}\sum_{i=1}^{n} \left[\left[1 - y_i f(\x_i) \right]_+ - \lambda \right]_+ + {k \over n} \lambda + \frac{1}{2C} \|f\|^2_K,
	\label{topk_hinge_v1}
\end{equation}
where $C > 0$ is the regularization factor.  Furthermore, the outer hinge function in \eqref{topk_hinge_v1} can be removed due to the following result.
\begin{lemma}\label{lemm:hing-hing} For $a \geq 0$, $b\geq 0$, there holds
	$\left[ \left[a - \ell \right]_+ - b \right]_+$ = $\left[a-b-\ell\right]_+$.
\end{lemma}
Proof of Lemma \ref{lemm:hing-hing} can be found in the \Apx. In addition, note that for any minimizer $(f_\bz, \gl_\bz)$ of \eqref{topk_hinge_v1}, setting $f(\x)=0, \gl=1$ in the objective function of \eqref{topk_hinge_v1}, we have $ {k\over n} \lambda_\bz  \le {1 \over n}\sum_{i=1}^{n} \left[\left[1 - y_i f_\bz(\bx_i) \right]_+ - \lambda_\bz \right]_+ + {k \over n} \lambda_\bz + \frac{1}{2C} \|f_\bz\|^2_K \le {k\over n}$, so we have $0 \le \gl_\bz \le 1$ which means that the minimization can be restricted to $0\le\gl\le 1.$  Using these results and introducing $\rho = 1-\lambda$, Eq.\eqref{topk_hinge_v1} can be rewritten as
\begin{equation}
	\min_{f \in \H_K, 0\le \rho \le 1} \  {1\over n} \sum_{i=1}^{n} [\rho - y_i f(\x_i)]_+ - {k \over n} \rho   + \frac{1}{2C} \|f\|^2_K.
	\label{topk_hinge_v2}
\end{equation}
The \atk-SVM objective generalizes many several existing SVM models.  For example, when $k = n$, it equals to the standard C-SVM \cite{cortes1995support}. When $C = 1$ and with conditions $K(\x_i,\x_i)\le 1$ for any $i$, \atk-SVM reduces to $\nu$-SVM \cite{scholkopf2000new} with $\nu={k\over n}$. Furthermore, similar to the conventional SVM model, writing in the dual form of \eqref{topk_hinge_v2} can lead to a convex quadratic programming problem that can be solved efficiently. See \Apx\: for more detailed explanations. 

{\bf Choosing $k$}. The number of top individual losses in the \atk loss is a critical parameter that affects the learning performance. In concept, using \atk loss will not be worse than using average or maximum losses as they correspond to specific choices of $k$. In practice, $k$ can be chosen during training from a validation dataset as the experiments in Section \ref{sec:4}. As $k$ is an integer, a simple grid search usually suffices to find a satisfactory value.  Besides, Theorem \ref{thm:fisher} in Section \ref{sec:analysis} establishes a theoretical lower bound for $k$ to guarantee reliable classification based on the Bayes error. If we have information about the proportion of outliers, we can also narrow searching space of $k$ based on the fact that \atk loss is the convex upper bound of the top-k loss, which is similar to \cite{shalev2016minimizing}.

%\red{is there a minimax SVM? and is it a special case of \atk-SVM? YY:  i do not know the answer.}
%
%\begin{proposition}\label{lemm:nu-repression}
%	Top-K SVM (\ref{topk_hinge_v3}) coincides with $\nu$-SVM algorithm \cite{scholkopf2000new}. \red{Yanbo, I remember $\nu$-SVM is a special case of our algorithm, should that be reflected here?}
%\end{proposition}
%

	%%%%%%%%%%%%%%%%%%%%%   Statistical analysis %%%%%%%%%%%%%%%%%%%%%
	% !TEX root =  ../draft.tex
%%%%%%%%%%%%%%%%%%%%%%%
\section{Statistical Analysis}\label{sec:analysis}

In this section, we address the statistical properties of the \atk objective in the context of binary classification. Specifically, we investigate the property of {\em classification calibration} \cite{bartlett2006convexity} of the \atk general objective, and  derive bounds for the misclassification error of  the \atk-SVM model in the framework of statistical learning theory (\eg\, \cite{bartlett2006convexity,de2005model,steinwart2008support,wu2006learning}).

%%%%%%%%%%%%%%%%%%%%%%
\subsection{Classification Calibration under \atk Loss}
\label{sec:31}
We assume the training data $\bz=\{(\x_i,y_i)\}_{i=1}^n$ are i.i.d. samples from an unknown distribution $p$ on $\X \times \{\pm 1\}$.  Let $p_\X$ be the marginal distribution of $p$ on the input space $\X$.
% 0-1 loss defined in the Intro
%Let the $0$-$1$ loss be given by the indicator function $\mbI_{[\cdot]}$ which is one if the argument is true and zero otherwise. 
Then, the misclassification error of a  classifier $f:\X \to \{\pm 1\}$ is denoted by $\cR(f) = \Pr( y \neq f(\x))=  \EX[\mbI_{yf(\x)\le 0}]$. The Bayes error is given by $\cR^\ast = \inf_{f} \cR(f),$ where the infimum is over all measurable functions. No function can achieve less risk than the Bayes rule $f_c(\x)=\sgn(\eta(\x) - {1 \over 2})$, where $\eta(\x) = \Pr(y=1|\x)$ \cite{devroye2013probabilistic}.  

In practice, one uses a surrogate loss $\ell: \R \to [0,\infty)$ which is convex and upper-bound the $0$-$1$ loss. The population  $\ell$-risk (generalization error) is given by $\cE_\ell(f) =  \EX[\ell(yf(x))]$.  Denote the optimal $\ell$-risk by $\cE^\ast_\ell = \inf_f \cE_\ell(f)$. A very basic requirement for using such a surrogate loss $\ell$ is the so-called {\em classification calibration} (point-wise form of Fisher consistency) \cite{bartlett2006convexity,lin}. Specifically, a loss $\ell$ is {\em classification calibrated} with respect to distribution $p$ if, for any $x$, the minimizer $f^\ast_\ell = \inf_{f}\cE_\ell(f)$ should have the same sign as the Bayes rule $f_c(\x)$, \ie, $\sgn(f^\ast_\ell(\x)) = \sgn(f_c(\x))$ whenever $f_c(\x) \neq 0$.  

An appealing result concerning the classification calibration of a loss function $\ell$ was obtained in \cite{bartlett2006convexity}, which states that $\ell$ is classification calibrated if $\ell$ is convex,  differentiable at $0$ and $\ell'(0)<0$.  In the same spirit, we investigate the classification calibration property of the \atk loss. Specifically, we first obtain the population form of the \atk objective using the infinite limit of \eqref{eq:atk}
\[
{1\over n}\sum_{i=1}^{n}\left[\ell(y_if(\x_i)) - \lambda \right]_{+} + {k \over n}\lambda \xrightarrow[n \to \infty]{{k \over n} \to \nu} \EX \lt[[ \ell(yf(\x))-\gl ]_+ \rt] + \nu\lambda.
\]
We then consider the optimization problem
\begeqn \label{eq:fisher-obj}  
(f^\ast,\gl^\ast) = \arg\inf_{f, \lambda\ge 0}  \EX \lt[[ \ell(yf(\x))-\gl ]_+\rt]  + \nu \gl,
\endeqn
where the infimum is taken over all measurable function $f: \X \to \R$. We say the \atk (aggregate) loss is classification calibrated with respect to $p$ if $f^\ast$ has the same sign as the Bayes rule $f_c$. 
The following theorem establishes such conditions.
% and implies how to choose parameter $k$ in \matk learning. 
%
\begth \label{thm:fisher}Suppose the individual loss $\ell:  \R  \to \R^+$ is convex, differentiable at $0$ and $\ell'(0)<0$. Without loss of generality, assume that $\ell(0)=1$. Let $(f^\ast,\gl^\ast)$ be defined in \eqref{eq:fisher-obj},
\begin{itemize}
\item[(i)] If $\nu> {\cE^\ast_\ell}$ then the \atk loss is classification calibrated.
\item[(ii)] If, moreover, $\ell$ is monotonically decreasing and the \atk aggregate loss is classification calibrated then $\nu \ge \int_{\eta(\bx)\neq {1\over 2}} \min( \eta(\x),1-\eta(\x) ) dp_\X(\x)$. \end{itemize}
\endth
The proof of Theorem \ref{thm:fisher} can be found in the \Apx. Part (i) and (ii) of the above theorem address respectively the sufficient and necessary conditions on $\nu$ such that the \atk loss becomes classification calibrated. Since $\ell$ is an upper bound surrogate of the $0$-$1$ loss, the optimal $\ell$-risk $\cE^\ast_\ell$ is larger than the Bayes error $\cR^\ast,$ \ie,\;  $\cE^\ast_\ell \ge \cR^\ast$. In particular, if the individual loss $\ell$ is the hinge loss then $\cE^\ast_\ell = 2\cR^\ast$.  Part (ii) of the above theorem indicates that the \atk aggregate loss is classification calibrated if $\nu=\lim_{n\to \infty} {k/n}$ is larger than the optimal generalization error $\cE^\ast_\ell$ associated with the individual loss. The choice of $k >  n \cE^\ast_\ell$ thus guarantees classification calibration, which gives a lower bound of $k$. This result also provides a theoretical underpinning of the sensitivity to outliers of the maximum loss (\atk loss with $k=1$).  If the probability of the set $\{x: \eta(x) = 1/2 \}$ is zero,  $\cR^\ast = \int_\X \min( \eta(\x),1-\eta(\x) ) dp_\X(\x) = \int_{\eta(\x)\neq {1/ 2} } \min( \eta(\x),1-\eta(\x) ) dp_\X(x)$. Theorem \ref{thm:fisher} indicates that in this case, if the maximum loss is calibrated, one must have ${1\over n} \approx \nu \geq R^\ast$. In other words, as the number of training data increases, the Bayes error has to be arbitrarily small, which is consistent with the empirical observation that the maximum loss works well under the well-separable data setting but are sensitive to outliers and non-separable data. 

% The samples belonging to the ``Bayes decision boundary" (\ie,  $\{x: \eta(x) = 1/2 \}$) are rare in practice and without loss of generalarity we assume that  the probability of the set $\{x: \eta(x) = 1/2 \}$ is zero.  In this case, the Bayes error \cite{devroye2013probabilistic} $\cR^\ast = \int_\X \min( \eta(x),1-\eta(x) ) dp_\X(x) = \int_{\eta(x)\neq {1/ 2} } \min( \eta(x),1-\eta(x) ) dp_\X(x)$.  When the maximum loss (\ie, $k=1$) is used then $\nu = \lim_{n\to \infty}{k/n} =0$. In this case, from part (iii) of Theorem \ref{thm:fisher},  classification calibration of such models  implies the Bayes error is zero.  This theoretical insight is consistent with the empirical practice that the minimax loss usually works well for the noiseless setting but not with heavy noisy case. 

%%%%%%%%%%%%%%%%%%%%%%%%%%
\subsection{Error bounds of \atk-SVM}

We next study the excess misclassification error of the \atk-SVM model \ie,\, $\cR(\sgn(f_\bz)) - \cR^\ast$. Let $(f_\bz, \rho_\bz)$ be the minimizer of the \atk-SVM objective \eqref{topk_hinge_v2} in the RKHS setting.  Let $f_\H$ be the minimizer of the generalization error over the RKHS space $\H_K$, \ie,\: $f_{\H} =  \argmin_{f\in \H_K}\cE_h(f)$, where we use the notation $\cE_h(f)= \EX\lt[[1 - yf(\x)]_+\rt]$ to denote the $\ell$-risk of the hinge loss. In the finite-dimension case, the existence of $f_\H$ follows from the direct method in the variational calculus, as $\cE_h(\cdot)$ is lower bounded by zero, coercive, and weakly sequentially lower semi-continuous by its convexity. For an infinite dimensional $\H_K$, we assume the existence of $f_\H$. We also assume that $\cE_h(f_\H) <1$ since even a na\"ive zero classifier can achieve $\cE_h(0) =1$.  Denote the approximation error by $\A(\H_K) = \inf_{f\in \H_K} \cE_h(f) - \cE_h(f_c) =  \cE_h(f_\H) - \cE_h(f_c)$, and let $\gk= \sup_{\x\in \X} \sqrt{K(\x,\x)}$. 
The main theorem  can be stated as follows.
\beg{theorem}\label{thm:analysis} Consider the \atk-SVM in RKHS \eqref{topk_hinge_v2}. 
For any $\gep \in (0, 1]$ and $\mu\in (0, 1-\cE_h(f_\H))$, choosing $k=  \lceil n (\cE_h(f_\H)+ \mu)\rceil$.   Then, it holds 
\[
\Pr\bigl\{  {\cal R}(\sgn(f_{\bz}))-\cR^\ast
\ge \mu   + \A(\H) +\gep +  {1+ C_{\gk,\H} \over \sqrt{n} \mu } \bigr\} \le 2\exp\bigl( -{n \mu^2 \gep^2 \over (1+ C_{\gk,\H})^2}\bigr),
\]
where $C_{\gk,\H} = \gk (2\sqrt{2C} + 4\|f_\H\|_K)$.
\end{theorem}
The complete proof of Theorem \ref{thm:analysis} is given in the \Apx. The main idea is to show that $\rho_\bz$ is bounded from below by a positive constant with high probability, and then bound the excess misclassification error  $\cR(\sgn(f^\ast_\bz)) - \cR^\ast$ by $\cE_h({f_\bz/\rho_\bz}) - \cE_h(f_c)$.
If $K$ is a universal kernel then $\A(\H_K)=0$ \cite{steinwart2008support}. In this case, let $\mu = \gep \in (0, 1-\cE_h(f_\H))$, then from Theorem \ref{thm:analysis} we have 
\[
\Pr\bigl\{  {\cal R}(\sgn(f_{\bz}))-\cR^\ast
\ge  2\gep +  {1+ C_{\gk,\H} \over \sqrt{n} \gep } \bigr\} \le 2\exp\bigl( -{n\gep^4 \over (1+ C_{\gk,\H})^2}\bigr),
\]
Consequently, choosing $C$ such that $\lim_{n\to\infty}  {C /n } = 0$, which is equivalent to $\lim_{n\to\infty}  {(1+ C_{\gk,\H})^2  /n } = 0$, then $ {\cal R}(\sgn(f_{\bz}))$ can be arbitrarily close to the Bayes error $\cR^\ast$, with high probability, as long as $n$ is sufficiently large.

%In the above theorem, the value $\mu$ is roughly the distinct between ${k\over n }$ and $\cE(f_\H)$ since, by the definition of $k=  \lceil n (\cE(f_\H)+ \mu)\rceil$,  $ \mu \le {k\over n} - \cE(f_\H) \le \mu+ {1\over n}$. In particular, let $\mu = \gep \in (0, 1-\cE(f_\H))$, then, with high probability $1-\gd$, there holds  $ {\cal R}(\sgn(f_{\bz}))-\cR_\ast \le  \A(\H)  + 2\bigl[   {  (1+ C_{\gk,H})^2 \log ({2\over \gd})\over n} \bigr]^{1\over 4} +   \bigl[{(1+ C_{\gk,H})^2 \over { n \log({2\over \gd})}}\bigr]^{1\over 4}$. Consequently, if the approximation error is zero,  which holds true for universal kernels, and choose $C$ satisfying $\lim_{n\to\infty}  {C \over n } = 0$, which is equivalent to $\lim_{n\to\infty}  {(1+ C_{\gk,H})^2  \over n } = 0$, and $k= \lceil n (\cE(f_\H)+ \gep)\rceil$, then we get from Theorem \ref{thm:analysis} that the consistency result holds true for \atk-SVM, \ie, $ {\cal R}(f_{\bz})$ converges to the Bayes error $\cR^\ast$ in probability. 

	%%%%%%%%%%%%%%%%%%%%%%% Experiments %%%%%%%%%%%%%%%%%%%%%%%%%%%%%%
	% !TEX root =  ../draft.tex
\section{Experiments}
\label{sec:4}

We have demonstrated that \atk loss provides a continuum between the average loss and the maximum loss, which can potentially alleviates their drawbacks. A natural question is whether such an advantage actually benefits practical learning problems. In this section, we demonstrate the behaviors of \matk learning coupled with different individual losses for binary classification and regression on synthetic and real datasets, with minimizing the average loss and the maximum loss treated as special cases for $k=n$ and $k=1$, respectively.  For simplicity, in all experiments, we use homogenized linear prediction functions $f(\x) = \w^T\x$ with parameters $\w$ and the Tikhonov regularizer $\Omega(\w) = {1 \over 2C} ||\w||^2$ , and optimize the \matk learning objective with the stochastic gradient descent method given in \eqref{eq:2}.

\textbf{Binary Classification:} We conduct experiments on binary classification using eight benchmark datasets from the UCI\footnote{\url{https://archive.ics.uci.edu/ml/datasets.html}} and KEEL\footnote{\url{http://sci2s.ugr.es/keel/datasets.php}} data repositories to illustrate the potential effects of using \atk loss in practical learning to adapt to different underlying data distributions. A detailed description of the datasets is given in \Apx. The standard individual logistic loss and hinge loss are combined with different aggregate losses. Note that average loss combined with individual logistic loss corresponds to the logistic regression model and average loss combined with individual hinge loss leads to the C-SVM algorithm \cite{cortes1995support}. 

For each dataset, we randomly sample 50\%, 25\%, 25\% examples as training, validation and testing sets, respectively. During training, we select parameters $C$ (regularization factor) and $k$ (number of top losses) on the validation set.  Parameter $C$ is searched on grids of $\log_{10}$ scale in the range of $[10^{-5}, 10^{5}]$ (extended when optimal value is on the boundary), and $k$ is searched on grids of $\log_{10}$ scale in the range of $[1, n]$. We use $k^\ast$ to denote the optimal $k$ selected from the validation set. 

We report the average performance over $10$ random splitting of training/validation/testing for each dataset with \matk learning objectives formed from individual logistic loss and hinge loss.  Table \ref{results:binary classification} gives their experimental results in terms of misclassification rate (results in terms of other classification quality metrics are given in \Apx). Note that on these datasets, the average loss consistently outperforms the maximum loss, but the performance can be further improved with the \atk loss, which is more adaptable to different data distributions. This advantage of the \atk loss is particularly conspicuous for datasets {\tt Monk} and {\tt Australian}.

\begin{table}
	\centering
	\footnotesize
	\renewcommand\arraystretch{1.1}

	\begin{tabular}{c|c c c || c c c}
		%binary classificaiton, logistic loss
		\hline 
		\multirow{2}{*}{} & \multicolumn{3}{c||}{Logistic Loss} & \multicolumn{3}{c}{Hinge Loss}
		\\ \cline{2-7}
		& Maximum & Average & AT$_{k^\ast}$ & Maximum & Average & AT$_{k^\ast}$ \\
		\hline	
		
		{\tt Monk}   
		& 22.41(2.95)	 & 20.46(2.02)	 & \textbf{16.76(2.29)}	     
		& 22.04(3.08)	 & 18.61(3.16)	 & \underline{17.04(2.77)}  \\
		
		{{\tt Australian}}
		& 19.88(6.64)	 & 14.27(3.22)	 & \textbf{11.70(2.82)} 	 
		& 19.82(6.56)	 & 14.74(3.10)	 & \underline{12.51(4.03)}  \\
		
		{\tt Madelon}
		& 47.85(2.51)	 & 40.68(1.43)	 & \textbf{39.65(1.72)} 	
		& 48.55(1.97)	 & 40.58(1.86)	 & 40.18(1.64) 	\\
		
		{\tt Splice}
		& 23.57(1.93)	 & 17.25(0.93)	 & \textbf{16.12(0.97)}
		& 23.40(2.10)	 & \underline{16.25(1.12)}	 & \underline{16.23(0.97)}  \\
		
		{\tt Spambase}
		& 21.30(3.05)	 & 8.36(0.97)	 & 8.36(0.97)
		& 21.03(3.26)	 & \textbf{7.40(0.72)}	 & \textbf{7.40(0.72)}  \\
		
		{\tt German}
		& 28.24(1.69)	 & 25.36(1.27)	 & \textbf{23.28(1.16)}
		& 27.88(1.61)	 & \underline{24.16(0.89)}	 & \underline{23.80(1.05)}  \\
		
		{\tt Titanic}
		& 26.50(3.35)	 & 22.77(0.82)	 & 22.44(0.84)
		& 25.45(2.52)	 & 22.82(0.74)	 & \textbf{22.02(0.77)}  \\
		
		{\tt Phoneme}   
		& 28.67(0.58)	 & 25.50(0.88)	 & 24.17(0.89)  
		& 28.81(0.62)	 & \textbf{22.88(1.01)}	 & \textbf{22.88(1.01)}	\\
		
		\hline
	\end{tabular}
	\caption{\em \small Average misclassification rate (\%) of different learning objectives over $8$ datasets. The best results are shown in bold with results that are not significant different to the best results underlined.}
	\label{results:binary classification}
	~\vspace{-2.em}
\end{table}

\begin{figure}[t]
	%\begin{center}
	\begin{tabular}{c@{\hspace{0em}}c@{\hspace{0em}}c@{\hspace{0em}}c}
		\includegraphics[width=.24\textwidth]{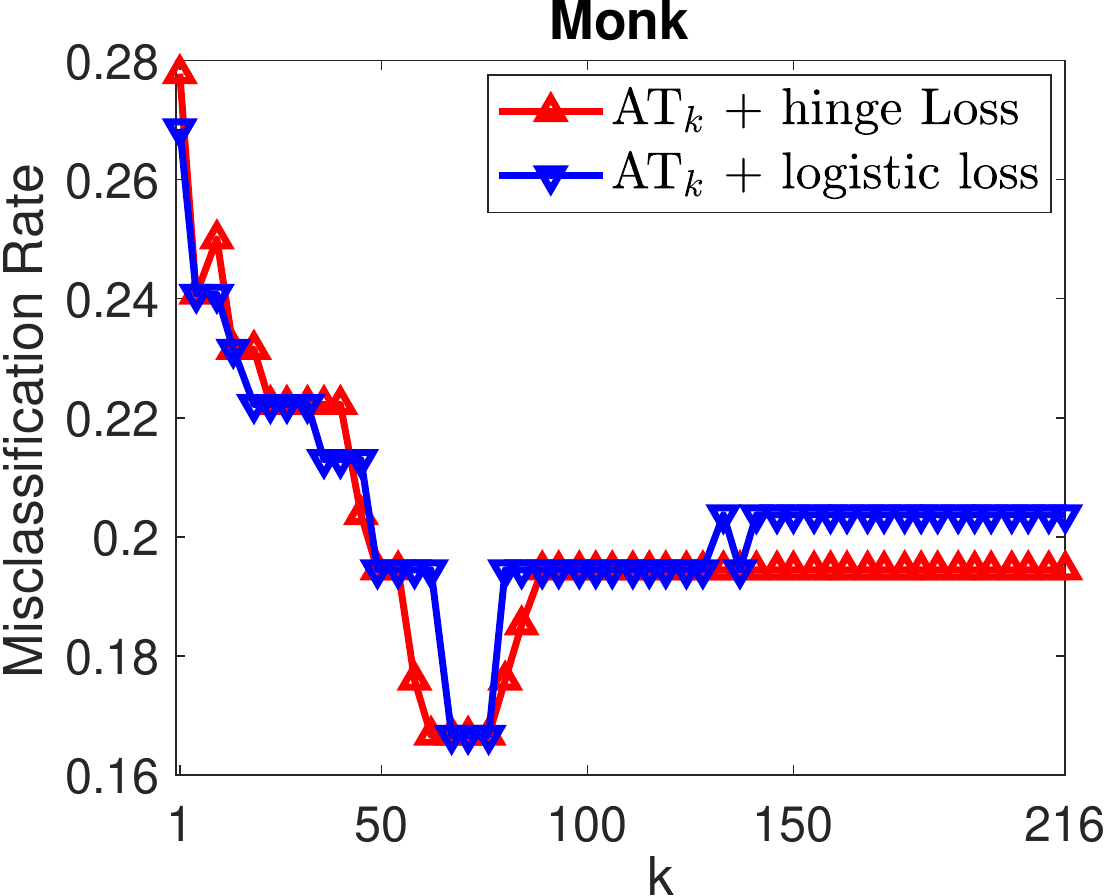} &
		\includegraphics[width=.24\textwidth]{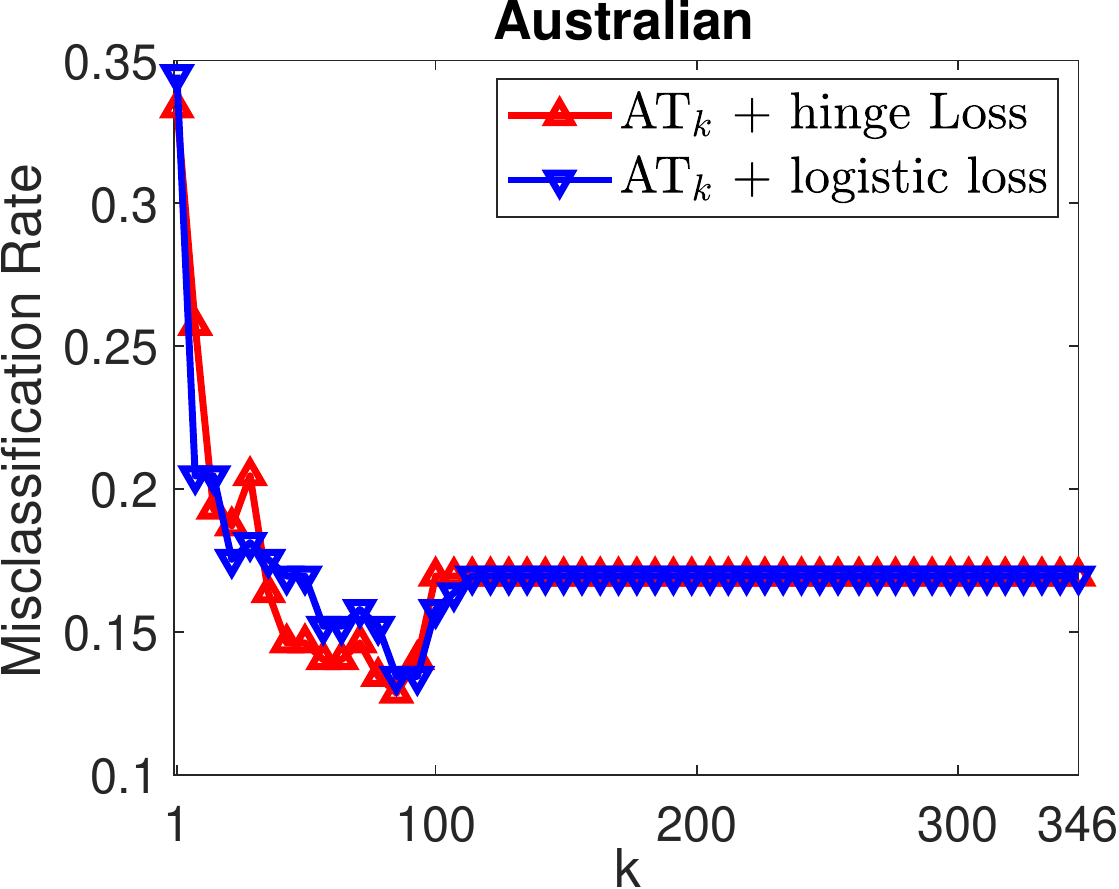}  &
		\includegraphics[width=.24\textwidth]{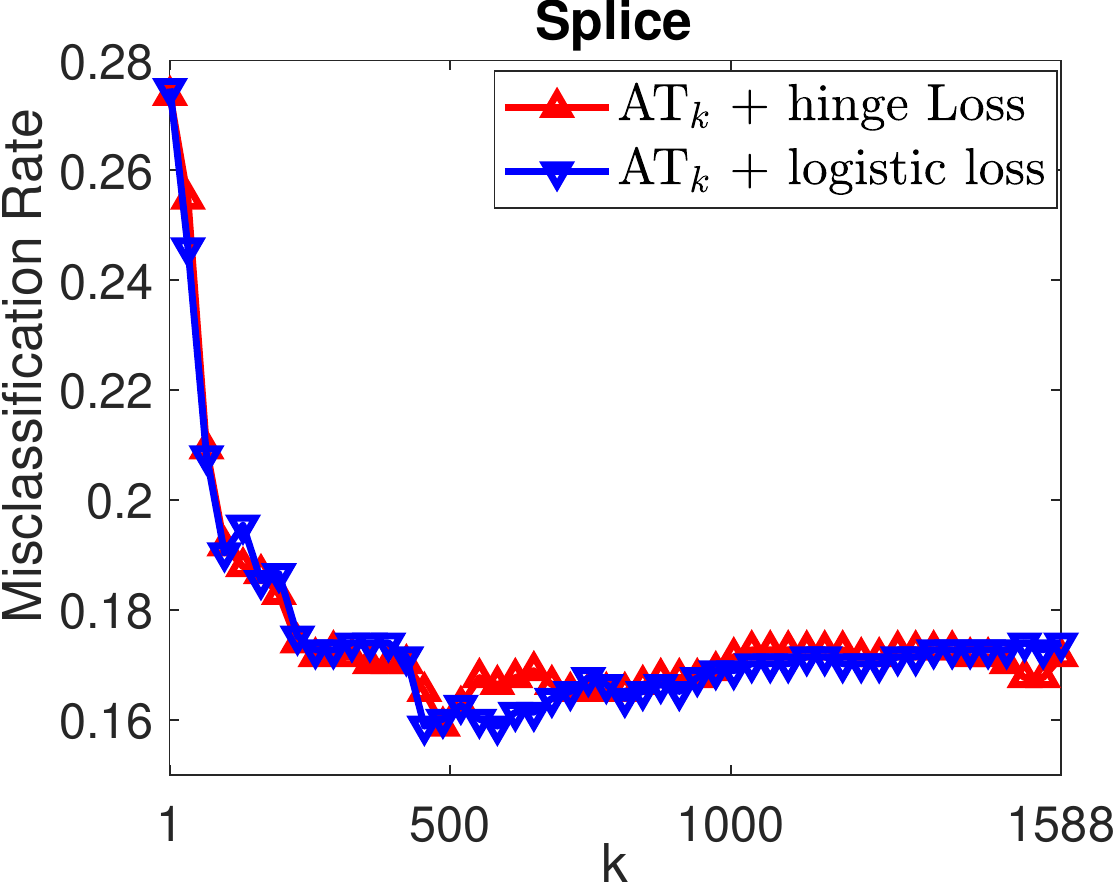} &
		\includegraphics[width=.24\textwidth]{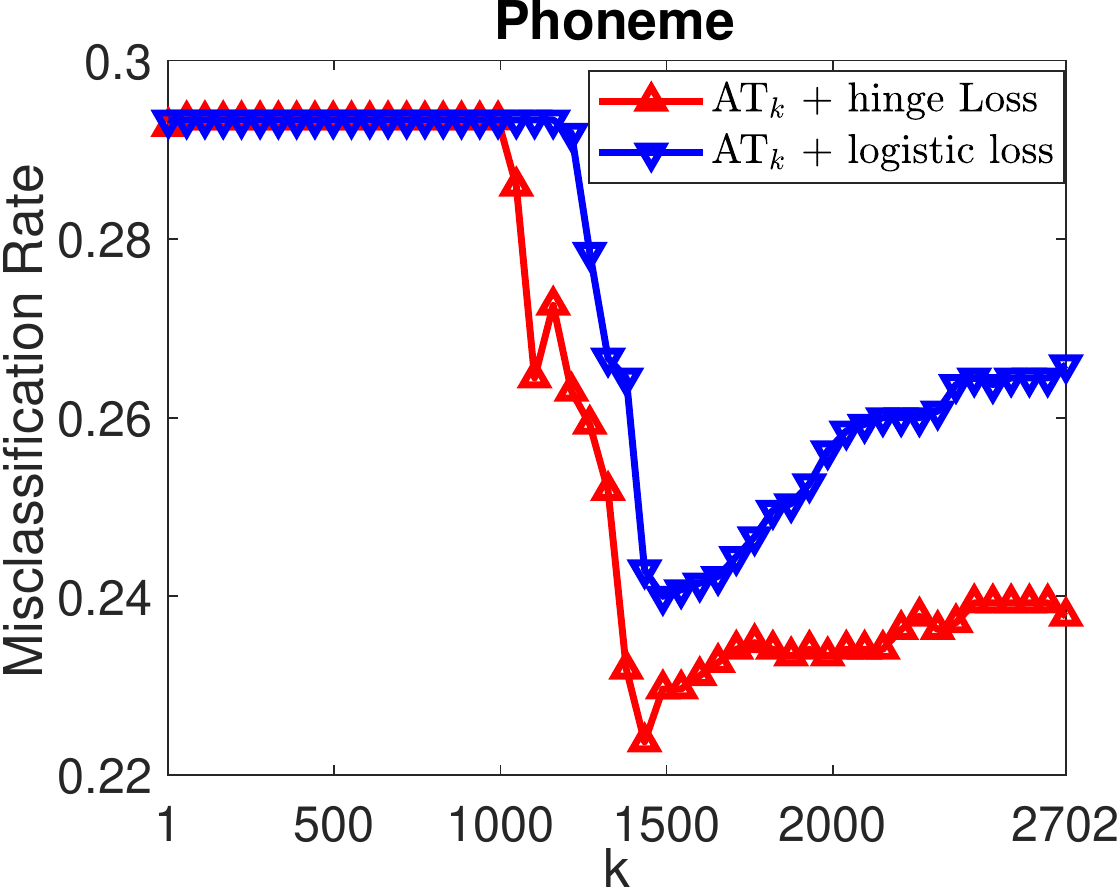} \\
	\end{tabular}
	%\end{center}
	~\vspace{-1em}
	\caption{\small \em Plots of misclassification rate on testing set vs. $k$ on  four datasets.}
	\label{fig:binary classification}
	~\vspace{-2.5em}
\end{figure}

To further understand the behavior of \matk learning on individual datasets, we show plots of misclassification rate on testing set vs. $k$ for four representative datasets in Fig.\ref{fig:binary classification} (in which $C$ is fixed to $10^2$ and $k \in [1,n]$). As these plots show, on all four datasets, there is a clear range of $k$ value with better classification performance than the two extreme cases $k=1$ and $k=n$, corresponding to the maximum and average loss, respectively. To be more specific, when $k=1$, the potential noises and outliers will have the highest negative effects on the learned classifier and the related classification performance is very poor. As $k$ increases, the negative effects of noises and outliers will reduce and the classification performance becomes better, this is more significant on dataset {\tt Monk}, {\tt Australian} and {\tt Splice}. However, if $k$ keeps increase, the classification performance may decrease (\eg, when $k=n$). This may because that as $k$ increases, more and more well classified samples will be included and the non-zero loss on these samples will have negative effects on the learned classifier (see our analysis in Section \ref{sec:interpretation}), especifically for dataset {\tt Monk}, {\tt Australian} and {\tt Phoneme}.

\textbf{Regression.} Next, we report experimental results of linear regression on one synthetic dataset ({\tt Sinc}) and three real datasets from \cite{chang2011libsvm}, with a detailed description of these datasets given in \Apx. The standard square loss and absolute loss are adopted as individual losses. Note that average loss coupled with individual square loss is standard ridge regression model and average loss coupled with individual absolute loss reduces to $\nu$-SVR \cite{scholkopf2000new}. 
We normalize the target output to $[0,1]$ and report their \emph{root mean square error (RMSE)} in Table \ref{results:regression}, with optimal $C$ and $k^*$ obtained by a grid search as in the case of classification (performance in terms of \emph{mean absolute square error (MAE) is given in \Apx}). Similar to the classification cases, using the \atk loss usually improves performance in comparison to the average loss or maximum loss.

\begin{table}[t]
	\centering
	\footnotesize
	\renewcommand\arraystretch{1.1}
	\setlength\tabcolsep{1.5pt}
	
	\begin{tabular}{c| c c c||c c c}
		%binary classificaiton, logistic loss
		\hline 
		\multirow{2}{*}{} & \multicolumn{3}{c||}{Square Loss} & \multicolumn{3}{c}{Absolute Loss}
		\\ \cline{2-7}
		& Maximum & Average & AT$_{k^\ast}$ & Maximum & Average & AT$_{k^\ast}$ \\		
		\hline
		{\tt Sinc}   
		& 0.2790(0.0449) & 0.1147(0.0060) & \textbf{0.1139}(0.0057) 
		& 0.1916(0.0771) & 0.1188(0.0067) & 0.1161(0.0060) \\
		%		\hline
		{\tt Housing}   
		& 0.1531(0.0226) & 0.1065(0.0132) & \textbf{0.1050}(0.0132) 
		& 0.1498(0.0125) & 0.1097(0.0180) & 0.1082(0.0189) 	  \\
		
		%		\hline
		{\tt Abalone}   
		& 0.1544(0.1012) & \underline{0.0800}(0.0026) & \textbf{0.0797}(0.0026)
		& 0.1243(0.0283) & 0.0814(0.0029) & 0.0811(0.0027)   \\
		
		{\tt Cpusmall}   
		& 0.2895(0.0722) & 0.1001(0.0035) & \textbf{0.0998}(0.0037) 
		& 0.2041(0.0933) & 0.1170(0.0061) & 0.1164(0.0062)   \\
		
		\hline
	\end{tabular}
	\caption{\small \em Average RMSE on four datasets. The best results are shown in bold with results that are not significant different to the best results underlined.}
	\label{results:regression}
		~\vspace{-3.5em}
\end{table}

	%%%%%%%%%%%%%%%%%%%%%%%  related work %%%%%%%%%%%%%%%%%%%%%%%%%%%%
	% !TEX root =  ../draft.tex
\section{Related Works}

%Meanwhile, the study of aggregate loss has also recently drawn attention in the ML community. 
%Standard average loss is most widely used as it is the unbiased approximation to the expected loss and is easy to optimize \cite{vapnik,smale2004shannon}. 
%Boosting algorithms are then proposed to assign a weight to each sample and emphasize samples that are misclassified by former weak classifiers. 
Most work on learning objectives focus on designing individual losses, and only a few are dedicated to new forms of aggregate losses. Recently, aggregate loss considering the order of training data have been proposed in {\em curriculum learning} \cite{bengio2009curriculum} and {\em self-paced learning} \cite{kumar2010self,spl2017}, which suggest to organize the training process in several passes and samples are included from {\em easy} to {\em hard} gradually. It is interesting to note that each pass of {\em self-paced learning} \cite{kumar2010self} is equivalent to minimum the average of the $k$ smallest individual losses, \ie, ${1 \over k} \sum_{i=n-k+1}^n \ell_{[i]}(f)$, which we term it as the {\em average bottom-k} loss in contrast to the average top-k losses in our case. In \cite{shalev2016minimizing}, the pros and cons of the maximum loss and the average loss are compared, and the top-k loss, \ie, $\ell_{[k]}(f)$, is advocated as a remedy to the problem of both. However, unlike the \atk loss, in general, neither the average bottom-k loss nor the top-k loss are convex functions with regards to the individual losses.

Minimizing top-$k$ errors has also been used in individual losses. For ranking problems, the work of \cite{rudin2009p,usunier2009ranking} describes a form of individual loss that gives more weights to the top examples in a ranked list.  In multi-class classification, the top-$1$ loss is commonly used which causes penalties when the top-1 predicted class is not the same as the target class label \cite{crammer2001algorithmic}. This has been further extended in \cite{lapin2015top,lapin2016loss} to the {\em top-$k$} multi-class loss, in which for a class label that can take $m$ different values, the classifier is only penalized when the correct value does not show up in the top $k$ most confident predicted values. As an individual loss, these works are complementary to the \atk loss and they can be combined to improve learning performance. 
	
	%%%%%%%%%%%%%%%%%%%%%%% Discussion & Conclusion  %%%%%%%%%%%%%%%%%
	% !TEX root =  ../draft.tex
\section{Discussion}

In this work, we introduce the {\em average top-$k$} (\atk) loss as a new aggregate loss for supervised learning, which is the average over the $k$ largest individual losses over a training dataset. We show that the \atk loss is a natural generalization of the two widely used aggregate losses, namely the average loss and the maximum loss, but can combine their advantages and mitigate their drawbacks to better adapt to different data distributions. We demonstrate that the \atk loss can better protect small subsets of hard samples from being swamped by a large number of easy ones, especially for imbalanced problems. Furthermore, it remains a convex function over all individual losses, which can lead to convex optimization problems that can be solved effectively with conventional gradient-based methods. We provide an intuitive interpretation of the \atk loss based on its equivalent effect on the continuous individual loss functions, suggesting that it can reduce the penalty on correctly classified data.  We further study the theoretical aspects of \atk loss on classification calibration and error bounds of minimum average top-$k$ learning for \atk-SVM. We demonstrate the applicability of minimum average top-$k$ learning for binary classification and regression using synthetic and real datasets.

There are many interesting questions left unanswered regarding using the \atk loss as learning objectives. Currently, we use conventional gradient-based algorithms for its optimization, but we are investigating special instantiations of \matk learning for which more efficient optimization methods can be developed. Furthermore, the \atk loss can also be used for unsupervised learning problems (\eg, clustering), which is a focus of our subsequent study. It is also of practical importance to combine \atk loss with other successful learning paradigms such as deep learning, and to apply it to large scale real life dataset. Lastly, it would be very interesting to derive error bounds of \matk with general individual loss functions.

	%%%%%%%%%%%%%%%%%%%%%%% acknowledgments  %%%%%%%%%%%%%%%%%
	% !TEX root =  ../draft.tex

\section{Acknowledgments}
This work was completed when the first author was a visiting student at SUNY Albany, supported by a scholarship from University of Chinese Academy of Sciences (UCAS). Siwei Lyu is supported by the National Science Foundation (NSF, Grant IIS-1537257) and Yiming Ying is supported by the Simons Foundation (\#422504) and the 2016-2017 Presidential Innovation Fund for Research and Scholarship (PIFRS) program from SUNY Albany. This work is also partially supported by the National Science Foundation of China (NSFC, Grant  61620106003) for Bao-Gang Hu and Yanbo Fan. 
	
	%%%%%%%%%%%%%%%%%%%%%%% Reference %%%%%%%%%%%%%%%%%%%%%%%%%%%%%%%%
	{\small
		\bibliographystyle{plain}
		\bibliography{refs_main}
	}
	
	%\end{document}
	
	\newpage
	%%%%%%%%%%%%%%%%%%%%%%% Appendix  %%%%%%%%%%%%%%%%%%%%%%%%%%%%%%%%
	% !TEX root =  ../draft.tex
\appendix 
\section{Proofs}
\subsection{Proofs of Lemma \ref{lem:1} and \ref{lemm:hing-hing}}

\noindent{\bf Proof of Lemma \ref{lem:1}}

Notice that $\sum_{i=1}^k x_{[i]}$ is the solution of the following linear programming problem
\begin{equation}
\max_{\bp} \ \bp^T \bx, \ \ \st \ \ \bp^T \bone = k, \bzero \leq \bp \leq \bone.
\label{prime_0}
\end{equation}
%Considering LP problem
%\begin{equation}
%\min_{\bp \in R^n} \ -\bp^T \bx, \ \ s.t. \ \ \bp^T \bone = k, \bzero \leq \bp \leq \bone.
%\label{prime_1}
%\end{equation}
%%
The Lagrangian of this linear programming problem is
\begin{equation}
L(\bp,\bu,\bv,\lambda) = -\bp^T \bx - \bv^T \bp + \bu^T(\bp - \bone) + \lambda (\bp^T \bone - k),
\end{equation}
where $\bu \geq 0$, $\bv \geq 0$ and t are Lagrangian multipliers. Taking its derivative w.r.t $\bp$ and set it to be $\bzero$, we have
$\bv = \bu -\bx +\lambda \bone$. Substituting this back into the Lagrangian to eliminate the primal variable, we obtain the dual problem of (\ref{prime_0}) as
\begin{equation}
\min_{\bu,\lambda} \ \bu^T\vec{1} + k\lambda, \ \st \ \bu \geq \bzero, \bu + \lambda \bone - \bx \geq \bzero.
\end{equation}
This further means that \begeqn\label{eq:app1}\sum_{i=1}^k x_{[i]}=\min_{\lambda} \Bigl\{ k\lambda + \sum_{i=1}^{n}\left[x_i - \lambda \right]_{+}  \Bigr\}.\endeqn 

%Next, assume $\lambda \in [x_{[j+1]}, x_{[j]}]$, then the objective function reduces to $\sum_{i=1}^j x_{[i]} + (k-j)\lambda$. This is a linear function \wrt~$\lambda$, and is decreasing for $j < k$ and increasing for $j > k$. As such, the optimal solution occurs at $j = k$.  \red{ for which part of lemma 1?}
%It is easy to see that $F(x,\gl) = k\lambda + \sum_{i=1}^{n}\left[x_i - \lambda \right]_{+} $ is jointly convex with respect to $(x,\gl)$ with $x = (x_1, x_2,\ldots,x_n)$.  
The convexity of $\sum_{i=1}^k x_{[i]}$ follows directly from \eqref{eq:app1} and the fact that the partial minimum of a jointly convex function is convex.  
%nonnegative Case s
Furthermore, it is easy to see that $\lambda = x_{[k]}$ is always one optimal solution for \eqref{eq:app1}, hence, for $x_i \ge 0, i = 1,\cdots, n$, there holds
\begeqn\label{eq:app2}\sum_{i=1}^k x_{[i]}=\min_{\lambda \ge 0} \Bigl\{ k\lambda + \sum_{i=1}^{n}\left[x_i - \lambda \right]_{+}  \Bigr\}.\endeqn 
\qed

\noindent{\bf Proof of Lemma \ref{lemm:hing-hing}}

Denote $g(\ell) = \left[ \left[a - \ell \right]_+ - b \right]_+.$ For any $a \geq 0, b\geq 0$, we have
$g(\ell) = 0 = \left[a-b-\ell \right]_+$ if  $\ell \geq a$. In the Case  of $\ell < a$, there holds $g(\ell) = \left[a-b-\ell \right]_+$. Thus $g(\ell) = \left[a-b-\ell \right]_+$ for any $a \geq 0, b\geq 0$.

\qed

\subsection{Proof of Theorem \ref{thm:fisher}} 

Note that $\ell:  \R  \to \R^+$ is convex, differentiable at $0$ and $\ell'(0)<0$ implies that $\ell(0) > 0$.
Hence, by normalization we can let $\ell(0)=1$. Indeed, the commonly used individual losses such as the least square loss $\ell (t) = {(1-t)^2} $, the hinge  loss $\ell(t) = (1-t)_+$, and the logistic loss  $ \ell(t) = \log_2(1+ e^{-t})$ satisfy the conditions $\ell'(0)<0$ and $\ell(0)=1$. The assumption in part (ii) of Theorem \ref{thm:fisher} implicitly assumes that $\cE^\ast_\ell \le 1$ because $\cE^\ast_\ell\le \cE_\ell(0)=1$. 

Since $(f^\ast,\gl^\ast)$ is a minimizer, then, by choosing $f=0$ and $\gl=\ell(0)=1$ there holds $  \EX [ \ell(yf^\ast(x))-\gl^\ast)_+ ]  + \nu \gl^\ast\le   \EX[(\ell(0) -\ell(0))_+] + \nu \ell(0) $ which implies that the minimizer $\gl^\ast$ defined in \eqref{eq:fisher-obj} must satisfy $0\le \gl^\ast \le \ell(0)=1$. 
 This means that the minimization over $\gl$ in \eqref{eq:fisher-obj} can be restricted to $0\le \gl\le \ell(0)=1.$ Let $\gb = 1 - \gl$ which implies that the minimization \eqref{eq:fisher-obj} is equivalent to  the following
\begeqn \label{eq:fisher-2}(f^\ast, \gb^\ast) = \arg\inf_{f,  0 \le \gb\le 1 }  \bigl\{\EX[  (\gb+ \ell(yf(\x)) - 1 )_+   ] - \nu\gb \bigr\} .\endeqn  
Let  $(f^\ast, \gb^\ast) $ be the minimizer.  We have, for any $f$ and  choosing $\gb = \ell(0)=1$, that  
$$  -\nu\gb^\ast \le \bigl\{\EX[  (\gb^\ast+ \ell(yf^\ast(x)) - 1 )_+   ] - \nu\gb^\ast  \le   \EX[  (1+ \ell(yf(\x)) - 1 )_+   ] - \nu  = \cE_\ell(f) - \nu.$$
This implies that $\nu \gb^\ast \ge \nu - \cE_\ell(f)$.  Since $f$ is arbitrary,     $\gb^\ast \ge  {\nu - \cE^\ast_\ell  \over \nu } > 0$ if $\nu > \cE^\ast_\ell$. Consequently, the above arguments show that $0\le \gl^\ast=1-\gb^\ast<1$ if $\nu> \cE^\ast_\ell.$ 

Now observe that $f^\ast  =  \arg\inf_{f} \bigl\{ \EX[ (\ell(yf(\x)) - \gl^\ast)_+ ] + \nu \gl^\ast \bigr\} = \arg\inf_{f}\bigl\{ \EX[ (\ell(yf(\x)) -\gl^\ast)_+ ] \bigr\}$.  Define $\phi(t) = (\ell(t) -\gl^\ast)_+$.  This means that    $f^\ast= \arg\inf_{f}\EX[\phi( yf(\x) ) ] $ for standard classification.  The result of Theorem 2 in \cite{bartlett2006convexity} states that that the loss $\phi$ is classification calibrated if $\phi$ is differentiable at $0$ and $\phi'(0) <0.$    Notice that $\gl^\ast< \ell(0)=1$ as proved above, which implies that $\phi$ is differentiable at $0$ and $\phi'(0) = \ell'(0) <0$. 	This shows that $f^\ast$ has the same sign as the Bayes rule $\sgn(\Pr(y=1|\x)-{1 \over 2})$ if $\nu> \cE^\ast_\ell.$  This completes the proof of the first part of the theorem.

We now move on to prove the proof of the second part of the theorem. To this end,  observe that $\gl^\ast = \arginf_{\lambda\ge 0}   \bigl\{ \EX [ \ell(yf^\ast(x))-\gl)_+ ]  + \nu \gl\bigr\}.$  Assume that $\gl^\ast=0$. Then, $f^\ast =f^\ast_\ell$ and choosing $f=0$ and $\gl=1=\ell(0)$ in the objective function of \eqref{eq:fisher-obj} implies that 
$ \nu= \EX[ (\ell(0) - 1)_+] +\nu \ge \EX [ \ell(y f^\ast(x) )- \gl^\ast)_+ ]  + \nu \gl^\ast  = \EX[\ell( y f^\ast_\ell(x) ) ] \ge \cR^\ast.$  Recall \cite{devroye2013probabilistic} that the Bayes error $\cR^\ast= \int_{\X} \min(\eta(\x), 1-\eta(\x)) \rho_\X(\x).$ This proves the Case  $\gl^\ast=0.$ 

%Consequently, $\nu\ge 1$ when $\gl^\ast=0$ which implies that $\nu=1$ (i.e. the average aggregate loss $k=n$) since $\mu$ is always less than one. 

Now it only suffices to prove the Case  of $\gl^\ast>0$. In this Case , by the first-order optimality condition, there exists a subgradient of $\EX [ \ell(yf^\ast(x))-\gl)_+ ]  + \nu \gl$ of the variable $\gl$ at $\gl^\ast$ equals to zero.  This implies that
$\EX [ h(x,y) ]  + \nu = 0,$ where $h(x,y)$ is some subgradient of $(\ell(yf^\ast(x))-\gl)_+$ with respect to $\gl$ at $\gl^\ast$. Notice that $h(x,y) \le -\mbI_{\ell(yf^\ast(x))>\gl^\ast}.$  Consequently, $\nu \ge \EX[\mbI_{\ell(yf^\ast(x))>\gl^\ast}] \ge \EX[\mbI_{\ell(yf^\ast(x))>\ell(0)}] $ since $\gl^\ast \le \ell(0)$ as proved in part (i).  Since we assume that $\ell$ is monotonically decreasing,  $\ell(y  f^\ast(x) )> \ell(0)$ is equivalent to $ yf^\ast(x) <0.$ The calibration of AT$_k$ models (i.e. $f^\ast$ has the sign as the Bayes rule)  implies that $y f^\ast(x) <0$ is equivalent to $ y (2\eta(x)-1) <0.$ Putting the above arguments together, we conclude that 
$\nu \ge \EX[\mbI_{y(2\eta(x)-1)<0}] = \int_{\eta(x)\neq {1/2}} \min(\eta(x),1-\eta(x)).$ This completes the proof of the theorem.

\subsection{Proof of Theorem \ref{thm:analysis}}

Steinwart \cite{steinwart2003optimal} derived the bounds for the excess misclassification error for $\nu$-SVM under the assumption that the kernel is {\em universal}, \ie, the RKHS  is dense in the space of continuous functions $\C(\X)$ under the uniform norm $\|\cdot\|_\infty$ (See \cite{steinwart2008support} for more details).  The proof there depends on Urysohn's lemma in topology which states any two disjoint closed subsets can be separated by a continuous function. In contrast, our result holds true without the assumption of universal kernels.  

To prove Theorem \ref{thm:analysis}, we need some technical lemmas.  We say the function $F:\displaystyle\dprod_{k=1}^m \Omega_k \rightarrow \mathbb{R}$ has bounded differences $\{ c_k \}_{k=1}^m $ if, for
all $1\le k \le m $, $$ \dmax_{z_1,\cdots
,z_k,z^\prime_{k}\cdots ,z_m}  |  F(z_1, \cdots
,z_{k-1},z_k,z_{k+1}, \cdots , z_m) -F(z_1, \cdots
,z_{k-1},z^\prime_k,z_{k+1}, \cdots , z_m) |\le c_k.
$$

\begin{lemma}\label{lem:McD}  (McDiarmid's inequality \citeapx{mcdiarmid1989method}) Suppose $f:\displaystyle\dprod_{k=1}^m \Omega_k \rightarrow
\mathbb{R}$ has bounded differences $ \{ c_k \}_{k=1}^m $ then ,
for all $\epsilon>0 $, there holds
$$ \Pr\biggl\{ F(\bz) -\mathbb{E}[F({\bf z})]\ge \epsilon\biggr\}\le
 e^{- \frac{2\epsilon^2}{\sum_{k=1}^m c_k^2}}.$$
\end{lemma}

We need to use the the Rademacher
average  and its contraction  property  \citeapx{bartlett2002rademacher,meir-zhang}.

\begdef Let $\mu$ be a probability measure on $\gO$ and $F$ be a
class of uniformly bounded functions. For every integer $m$,  the Rademacher
average over a set of functions F on Ω
$$R_m(F):= \EX_\mu \EX_\epsilon \Big\{{1\over{ {m}}} \dsup_{f\in F}
 \Big|\dsum_{i=1}^m \gs_i f(z_i) \Big|\Big\}$$ where
$\{z_i\}_{i=1}^m$ are independent random variables distributed
according to $\mu$ and $\{\gs_i\}_{i=1}^m$ are independent
Rademacher random variables, i.e.,
$\Pr(\gs_i=+1)=\Pr(\gs_i=-1)=1/2$.
\enddef

\begin{lemma}\label{lem:contr-prop}Let $F$ be a class of uniformly bounded real-valued
functions on $(\gO,\mu)$ and $m\in\N$. If for each $i\in \{1,
\ldots, m\}$, $\Psi_i: \R \to \R$ is a function  with a Lipschitz constant $c_i$, then for any $\{x_i\}_{i=1}^m$,
\begeqn \EX_\epsilon\Big( \dsup_{f\in F} \big|\dsum_{i=1}^m
\epsilon_i \Psi_i(f(x_i)) \big|\Big)
 \le 2   \EX_\epsilon\Big(  \dsup_{f\in F}
 \Big|\dsum_{i=1}^m c_i\epsilon_i
f(x_i) \big| \Big).\label{contr}\endeqn
\end{lemma}
Using the standard techniques involving  Rademacher averages \citeapx{bartlett2002rademacher}, one can get the following estimation.   For completeness, we give a self-contained proof.  Let the empirical error related to the hinge loss  be denoted by 
$\cE_{h,\bz}(f) = {1\over n} \sum_{i=1}^n (1-yf(\x_i))_+.$
\begin{lemma}\label{lem:rad-est} For any $\gep>0$, there holds
$$  \Pr\biggl\{ \dsup_{\|f\|_K \le R} \cE_h(f) - \cE_{h,\bz}(f) \ge \gep +  {2 \gk R \over \sqrt{n}} \biggr\} \le e^{- {2n\gep^2 \over (1+ \gk R)^2} }.   $$
\end{lemma}
\begin{proof}  Let $F(z) = \dsup_{\|f\|_K \le R} [\cE_h(f) - \cE_{h,\bz}(f)].$ Observe, for any $x,y$, that $(1-yf(x))_+ \le 1+ |f(x)| \le 1+ |\langle K_x, f \rangle_K| \le 1+ \|f\| {\langle K_x, K_x \rangle_K}^{1\over 2} = 1+ \|f\|_K \sqrt{K(x,x)}\le  \gk R. $ Then, one can easily get that the bounded differences are $c_k = {1+ \gk R\over n} $ for any $1\le k\le n.$
By the McDiarmid inequality, we have
$${\Pr} \biggl\{ \dsup_{\|f\|_K \le R} [\cE_h(f) - \cE_{h,\bz}(f)] \ge  \EX_\bz\dsup_{\|f\|_K \le R} [\cE_h(f) - \cE_{h,\bz}(f)] +  \gep \biggr\}\ge
 \exp\Bigr\{- {2n\gep^2 \over (1+\gk R)^2 }\Bigl\}.$$
Let $\bz' = \{z_1',z_2',\ldots, z_n'\}$ be i.i.d. copies of $\bz.$ Then,
$$\EX_\bz\dsup_{\|f\|_K \le R} [\cE_h(f) - \cE_{h,\bz}(f)]  = \EX_\bz[\dsup_{\|f\|_K \le R} [\EX_{\bz'}(\cE_{\bz'}(f)) - \cE_{h,\bz}(f)] \le  \EX_\bz\EX_{\bz'}\dsup_{\|f\|_K \le R} [\cE_{\bz'}(f) - \cE_{h,\bz}(f)].$$
By standard symmetrization techniques \citeapx{bartlett2002rademacher}, for any Rademacher variables $\{\gs_i: i=1,\ldots,n\}$, we have that
\begin{align*}& \EX_\bz\EX_{\bz'}\dsup_{\|f\|_K \le R} [\cE_h(f) - \cE_{h,\bz}(f)]  = \EX_\bz\EX_{\bz'}\EX_{\gs}\dsup_{\|f\|_K \le R}[ {1\over n }\dsum_{i=1}^n \gs_i((1-y'_i f(x'_i))_+ - (1- y_i f(x_i))_+) ]\\
&  = 2 \EX_\bz\EX_{\gs}\dsup_{\|f\|_K \le R}\bigl[ {1\over n }\dsum_{i=1}^n \gs_i (1- y_i f(x_i))_+ \bigr] \le 2 \EX_\bz\EX_{\gs}\dsup_{\|f\|_K \le R}  {1\over n }\Big| \dsum_{i=1}^n \gs_i (1- y_i f(x_i))_+ \Big|.
\end{align*}
Let $\Phi_i (t) = (1-y_i t)_+$ which has Lipschitz constant $1$. By the contraction property of Rademacher averages,
\beg{align*} \EX_{\gs}\dsup_{\|f\|_K \le R}  {1\over n }\Big| \dsum_{i=1}^n \gs_i (1- y_i f(x_i))_+ \Big|  & \le \EX_{\gs}\dsup_{\|f\|_K \le R}{1\over n } \Big| \dsum_{i=1}^n \gs_i f(x_i) \Big|      = \EX_{\gs}\dsup_{\|f\|_K \le R}\Big|   \langle {1\over n }  \dsum_{i=1}^n \gs_i K_{x_i},  f\rangle \Big| \\
& \le  \EX_{\gs}\dsup_{\|f\|_K \le R}  \|{1\over n }  \dsum_{i=1}^n \gs_i K_{x_i}\|_K \|f\|_K  = R \: \EX_{\gs} \Big[\|{1\over n } \dsum_{i=1}^n \gs_i K_{x_i}\|_K\Big]\\
& \le  R\:  \Big[ \EX_{\gs} \|{1\over n } \dsum_{i=1}^n \gs_i K_{x_i}\|^2_K\Big]^{1\over 2} \le  {R\over n} \Big[ \dsum_{i=1}^n K(x_i,x_i)\Big]^{1\over 2} \le {\gk R \over \sqrt{n}}.
\end{align*}
Putting all the above estimations together yields the desired result.  This completes the proof of the lemma.
\end{proof}

We also need the H\"oeffding's inequality stated as follows. 
\beg{lemma}\label{lemm:hfding}
    Let $\xi$ be a random variable and, for any $i\in [m]$, $a_i\le\xi \le b_i.$ Then, for any $\gep>0$, there holds
    $${\Pr}\biggl\{{1\over m} \dsum_{i=1}^m \xi_i- E\xi
    \ge \epsilon \biggr\}
    \le \exp\biggl\{- {m \epsilon^2 \over {2 M^2}}\biggr\}.$$
  \end{lemma}

To prove the main theorem, we need to establish a lower bound for $\rho_\bz.$  Denote $\gk = \sup_{x\in \X} \sqrt{K(x,x)}. $
\beg{lemma} \label{lem:rhoinq}For  $\mu \in (0, 1-\cE_h(f_\H))$,  let $ \lceil n(\cE_h(f_\H) + \mu)\rceil \le k \le n$, then we have
$$\Pr\biggl\{\bz\in \Z^n:   {\|f_\bz\|_K \over \rho_\bz}\le   { 2k\over n}  \max\Bigl(\sqrt{2C \over \mu },  {2 \|f_\H\|_K\over \mu}  \Bigr) \biggr\} \ge 1- \exp\Bigl\{ - {  n \mu^2 \over 2(1+ \gk \|f_\H\|_K)^2}\Bigr\}.$$
\end{lemma}

\begin{proof} Since $(f_\bz,\rho_\bz)$ is a minimizer of  formulation \eqref{topk_hinge_v2},   for any $0<\rho \le 1$ there holds
\begin{align*}{1\over n}    \sum_{ i =1}^n  (\rho_\bz - y_i f_\bz(x_i))_+  - {k\over n } \rho_\bz    + {1 \over 2C}\|f_\bz\|^2_K & \le  {1\over n}\displaystyle\sum^{n}(\rho-y_i\rho f_\H(x_i))_+  - {k\over n }\rho+  {1\over 2C}\|\rho f_\H\|^2_K\\
& = \rho \cE_{h,\bz} (f_\H) - {k\over n }\rho + \frac{\rho^2}{2C}\|f_\H\|^2_K. \numberthis \label{eq:inter1}
    \end{align*}
This implies, for any $0<\rho\le 1$, that $${k\over n } \rho_\bz \ge  -\rho \cE_{h,\bz} (f_\H) + {k\over n } \rho - \frac{\rho^2}{2C}\|f_\H\|^2_K.$$
Applying the Hoeffding inequality (Lemma \ref{lemm:hfding})  yields that
\begeqn\label{eq:hfinq}\Pr\Bigl\{\cE_{h,\bz}(f_\H) - \cE_h(f_\H) \le  {\mu \over 2} \Bigr\} \le 1- \exp\Bigl\{ - {  n \mu^2 \over 2(1+ \gk \|f_\H\|_K)^2}\Bigr\}.\endeqn
Then, on the event \;$\mathcal{U} =  \bigl\{\bz\in \Z^n: \cE_{h,\bz}(f_\H) - \cE_h(f_\H) \le {\mu\over 2} \bigr\},$ we have $ -\rho \cE_{h,\bz} (f_\H) + {k\over n}\rho - \frac{\rho^2}{2C}\|f_\H\|^2_K\ge  {\rho ({k\over n}- \cE (f_\H) - {\mu\over 2}) }- \frac{\rho^2}{2C}\|f_\H\|^2_K\ge \rho {\mu\over 2}- \frac{\rho^2}{2C}\|f_\H\|^2_K.$   Define $g(\rho) = {\rho \mu\over 2}- \frac{\rho^2}{2C}\|f_\H\|^2_K.$ It is easy to observe that $$\dmax_{0<\rho\le 1} g(\rho) \ge \left\{\beg{array}{ll}   {C \mu^2 \over 8 \|f_\H\|^2_K},  &  C \mu \le 2\|f_\H\|_K^2, \\
{\mu \over 4 }, &   C \mu > 2\|f_\H\|_K^2.
\end{array}\right.$$
Consequently, on the event $\mathcal{U}$,  there holds 
 \begeqn\label{eq:inequ}  \rho_\bz\ge {n\over k}\dmax_{0<\rho\le 1} g(\rho) \ge  {n\over k} \min\bigl({\mu\over 4},   {C \mu^2 \over 8\|f_\H\|_K^2}\bigr).\endeqn

  By choosing $\rho=0$ in \eqref{eq:inter1}, there holds ${\|f_\bz\|_K^2 \over \rho_\bz} \le {2 C k\over n }$.  Combining these estimation together, on the event $\U$ there holds
$$ {\|f_\bz\|_K \over \rho_\bz} \le \sqrt{\|f_\bz\|_K^2 \over \rho_\bz }\sqrt{1\over \rho_\bz} \le   { 2k\over n}  \max\biggl(\sqrt{2C \over \mu },  {2 \|f_\H\|_K\over \mu}  \biggr).$$
This completes the proof of the lemma. \end{proof}

With all the above technical lemmas, we are ready to  prove Theorem \ref{thm:analysis}.

\noindent {\bf Proof of Theorem \ref{thm:analysis}}.  We will use the relationship between the excess misclassification error and generalization error \citeapx{zhang2004statistical}, \ie\: for any $f: \X \to \R$, there holds \begeqn \label{eq:comparison-inq}\cR(\sgn(f)) - \cR(f_c) \le \cE_h(f) - \cE_h(f_c). \endeqn
Let $\U_1$ be the event such that the inequality in Lemma \ref{lem:rhoinq} is true, \ie\: $\U_1 = \bigl\{\bz\in \Z^n: {\|f_\bz\|_K \over \rho_\bz}\le   { 2k\over n}  \max\bigl(\sqrt{2C \over \mu },  {2 \|f_\H\|_K\over \mu} \bigr) \bigr\}.$   On the event $\U_1$, noting that $0<\mu\le 1$ we have  that ${\|f_\bz\|_K \over \rho_\bz}\le R_{C, \mu}:= {2 \sqrt{2C}+ 4 \|f_\H\|_K \over \mu}.$

Now considering the sample $\bz\in \U_1, $ using  \eqref{eq:comparison-inq} we have
\begeqn\label{eq:inter2}{\cal R}(\sgn(f_{\bf z}))-{\cal R}(f_c)
\le \cE_h\bigl({f_{\bf z} \over  \rho_{\bf z}}\bigr)-\cE(f_c)\le
\cE_h\bigl(\frac{f_{\bf z}}{\rho_{\bf z}}\bigr)-\cE_{h, \bf
z}(\frac{f_{\bf z}}{\rho_{\bf z}})+ \cE_{h, \bf
z}(\frac{f_{\bf z}}{\rho_{\bf z}})-\cE(f_c)\endeqn
By the definition of the minimizer $(\rho_\bz,f_\bz)$, there holds ${1\over n}\sum_{ i =1}^n  (\rho_\bz - y_i f_\bz(x_i))_+  - {k\over n}\rho_\bz    + {1 \over 2C}\|f_\bz\|^2_K  \le 0$ which means that  ${1\over n}\sum_{ i =1}^n  (\rho_\bz - y_i f_\bz(x_i))_+  \le {k\over n} \rho_\bz.$  Equivalently,  $ \cE_{h,\bz}\bigl({f_\bz\over \rho_\bz}\bigr) \le {k\over n}$ on the event $\U_1$.  This combines with \eqref{eq:inter2} implies, on the event $\U_1$, that
\beg{align*} {\cal R}(\sgn(f_{\bf z}))-{\cal R}(f_c) & \le \cE_h\bigl(\frac{f_{\bf z}}{\rho_{\bf z}}\bigr)-\cE_{h, \bf
z}(\frac{f_{\bf z}}{\rho_{\bf z}})+ ({k\over n} -  \cE_h(f_\H))  +  \cE_h(f_\H)- \cE(f_c) \\
&  \le   \dsup_{\|f\|_K \le R_{C,\mu} }  \Bigl[\cE_h\bigl(f\bigr)-\cE_{h, \bz}(f)\Bigr]+ ({k\over n} -  \cE_h(f_\H))  +  \inf_{f\in \H_K} \cE_h(f) -\cE_h(f_c) \\
&  \le \dsup_{\|f\|_K \le R_{C,\mu} }  \Bigl[\cE_h\bigl(f\bigr)-\cE_{h, \bz}(f)\Bigr]+ ({k\over n} -  \cE_h(f_\H))  +  \A(\H_K)\\
& \le  \dsup_{\|f\|_K \le R_{C,\mu} }  \Bigl[\cE_h\bigl(f\bigr)-\cE_{h, \bz}(f)\Bigr]+ \mu + {1\over n}  +  \A(\H_K),
\end{align*}
where the last inequality follows from the fact, by the definition $k=k(n) = \lceil n(\cE_h(f_\H)+ \mu)\rceil$, that $\cE_h(f_\H) + \mu\le {k\over n} \le \cE_h(f_\H) + \mu+ {1\over n}.$
Therefore,
\beg{align*} & \Pr \biggl\{\bz\in \Z^n:   {\cal R}(\sgn(f_{\bf z}))-{\cal R}(f_c)  \ge  {\mu} + {1\over n} + \A(\H) +\gep +  {2\gk R_{C,\mu} \over \sqrt{n}}\biggr\} \\ & \le   \Pr(\U_1^c) +  \Pr\biggl\{\bz\in \U_1:  \dsup_{\|f\|_K  \le R_{C,\mu} }  \bigl[\cE_h\bigl(f\bigr)-\cE_{h, \bz}(f)\bigr] \ge \gep + {2\gk R_{C,\mu} \over \sqrt{n}} \biggr\} \\
& \le   \exp\biggl( -{n \mu^2 \over 2(1+ \gk \|f_\H\|_K)^2}\biggr)+ \Pr\biggl\{\bz:   \dsup_{\|f\|_K  \le R_{C,\mu} }  \bigl[\cE_h\bigl(f\bigr)-\cE_{h, \bz}(f)\bigr] \ge \gep +  {2\gk R_{C,\mu} \over \sqrt{n}}  \biggr\}\\
& \le \exp\biggl( -{n \mu^2 \over 2(1+ \gk \|f_\H\|_K)^2}\biggr) + \exp\biggl( - {2n\gep^2 \over (1+ \gk R_{C,\mu})^2}\biggr)\\
& \le 2\exp\biggl(-{ n\gep^2 \mu^2 \over (1+2\gk\sqrt{2C} + 4\gk \|f_\H\|_K)^2}\biggr).
\end{align*}
Here, the second to last inequality follows from Lemma \ref{lem:rad-est} which is the standard estimation for Rademacher averages \citeapx{bartlett2002rademacher}.
\qed

\section{Examples of \atk loss coupled with different individual losses}
The proposed \atk loss is quite general and can be combined with different existing individual losses. An interesting phenomenon is that \atk with hinge loss and absolute loss have a close relations to the well-known $\nu$-SVM and $\nu$-SVR that proposed in \cite{scholkopf2000new}, respectively. Specifically, we have 

\begprop
Under conditions $C=1$ and $K(\x_i,\x_i) \le 1$ for any $i$, \atk-SVM \eqref{topk_hinge_v2} reduces to $\nu$-SVM with $\nu= {k\over n}$.
\endprop

\begin{proof}
Recall \cite{scholkopf2000new} that the primal problem of the $\nu$-SVM without the bias term $b$ is formulated by 
\begeqn\label{eq:nu-svm}
\min_{f\in\H_K, \rho\ge 0} {1\over n} \sum_{i=1}^n \left[\rho- y_i f(\x_i) \right]_+  -\nu\rho + {1\over 2} \|f\|_K^2,
\endeqn
where $\nu \in [0,1]$ is a scalar. Its dual is given by $$\left\{\begin{array}{cl} \min_{\ga}  &  {1\over 2}\dsum_{i,j=1}^n \ga_i \ga_j y_i y_j K(x_i,x_j)\\
\hbox{s.t.} & 0\le\ga_i\le {1\over n}, \forall i =1,2,\ldots,n\\
& \dsum_{i=1}^n \ga_i \ge \nu. 
\end{array}\right.$$  
The KKT conditions implies, for any optimal solution $\ga^\ast$ of the dual and any optimal solution $(f_\bz,\rho_\bz)$ of the primal, there holds, for the support vectors $\x_i$ with $0<\ga^\ast_i < {1\over n}$, that $\rho_\bz = y_i \sum_{j=1}^n \ga^\ast_j y_j K(\x_i, \x_j)$.  If one assumes that $K(\x_i,\x_i)\le 1$ for all $i$, then $|K(x_i,x_j)| = |\langle K_{\x_i}, K_{\x_j} \rangle_K| \le \sqrt{K(\x_i,\x_i)}\sqrt{K(\x_j,\x_j)}\le 1$. Therefore,  
$$ \rho_\bz \le |y_i \sum_{j=1}^n \ga^\ast_j y_j K(\x_i, \x_j) | \le \sum_{j=1}^n \ga^\ast_j \le 1,$$
where the last inequality follows from the fact that $\ga_j \le {1\over n} $ for all $j.$ Consequently, in the minimization of \eqref{eq:nu-svm} we can restrict to $\rho\le 1$ which implies that the \atk-SVM \eqref{topk_hinge_v2} with $C=1$ is reduced to $\nu$-SVM with $\nu = {k\over n}.$ \end{proof}
Besides, the dual formulation of \atk-SVM  \eqref{topk_hinge_v2} can be easily derived as 
$$\left\{\begin{array}{cl} \min_{\ga}  &  {1\over 2}\dsum_{i,j=1}^n \ga_i \ga_j y_i y_j K(x_i,x_j) - \dsum_{i=1}^n \ga_i \\
\hbox{s.t.} & 0\le\ga_i\le {C\over n}, \forall i =1,2,\ldots,n\\
& \dsum_{i=1}^n \ga_i \le { Ck \over n}. 
\end{array}\right.$$  
This leads to a convex quadratic programming problem for \atk-SVM and can be solved efficiently.

\begprop
\matk model \eqref{eq:1} coupled with absolute loss in the RKHS setting becomes $\nu$-SVR with $\nu = {k\over n}$.
\endprop

\begin{proof}	
	Recall \cite{scholkopf2000new} that the primal problem of the $\nu$-SVR without the bias term $b$ in RKHS is formulated by 
	\begeqn\label{eq:nu-svr}
	\min_{\w, \lambda\ge 0} {1\over n} \sum_{i=1}^n \left[ |y_i - f(x_i)| -\lambda \right]_+ + \nu \lambda + {1\over 2C} \|f\|_K^2,
	\endeqn
	where $\nu \in [0,1]$ is a scalar. It is easy to see in the setting of RKHS that, with individual absolute loss (\ie, $\ell(f(\x_i),y_i) = |y_i - f(\x_i)|$) and $\Omega(\w) = {1 \over 2C} \|f\|_K^2$, \matk model \eqref{eq:1} becomes 
	\begeqn\label{eq:atk-svr}
	\min_{\w, \lambda\ge 0} {1\over n} \sum_{i=1}^n \left[ |y_i - f(\x_i)| -\lambda \right]_+ + {k\over n} \lambda + {1\over 2C} \|f\|_K^2,
	\endeqn
	We name model \eqref{eq:atk-svr} as \atk-SVR for brevity. It is straightforward that \atk-SVR is exactly the $\nu$-SVR with $\nu = {k \over n}$. 
		
The above propositions provide new perspectives to understand the success of $\nu$-SVM and $\nu$-SVR. That is, through ``shifting down'' the original individual hinge loss and absolute loss and truncating them at zero, the penalty of correctly classified samples that are ``far enough'' from classification boundary in classification and the penalty of samples that are ``close enough'' to the regression tube in regression will be zero, which enables the model to put more effort to misclassified samples or samples that are ``too far" to the regression tube. Besides, the good properties of $\nu$ in $\nu$-SVM and $\nu$-SVR that derived in \cite{scholkopf2000new} can be extended to $k$ in \atk-SVM and \atk-SVR directly. 
For example, for \atk-SVM with conditions $C=1$ and $K(\x_i,\x_i) \le 1$ and \atk-SVR, $k$ is a lower bound on the number of support vectors and is an upper bound on the number of margin errors. Due to its directness, we refer to \cite{scholkopf2000new} for their proofs. This can also help us select $k$ in \atk-SVM and \atk-SVR.
	
\end{proof}
\label{key}

\section{Toy examples for effects of different aggregate losses} 
We illustrate the behaviors of different aggregate losses using binary classification on 2D synthetic data examples. We generate six different datasets (Fig. \ref{fig:synthetic full}). Each dataset consists of 200 samples sampled from Gaussian distributions with distinct centers and variances. We use linear classifier and consider different aggregate losses combined with individual logistic loss and individual hinge loss. The learned linear classifiers and the misclassification rate of \atk vs. $k$ are shown in Fig. \ref{fig:synthetic full}. The left panel in Fig. \ref{fig:synthetic full} (\ie, (a1-a6) and (b1-b6)) refers to the results of aggregate losses combined with individual logistic loss and the right panel (\ie, (c1-c6) and (d1-d6)) refers to the results of aggregate losses combined with individual hinge loss.

%Case  1
\textbf{Case 1.} The first row in Fig. \ref{fig:synthetic full} represents an ideal situation where there is no outliers and the $+$ samples and $-$ samples are well distributed and linear separable. In this Case , all aggregate losses with both logistic loss and hinge loss can get perfect classification results. This is also verified in Fig. \ref{fig:synthetic full}~(b1) and Fig. \ref{fig:synthetic full}~(d1) that the misclassification rate is zero for \atk with all $k$. 

%Case  2
\textbf{Case 2.} In the second row, there exists an outlier in the $+$ class (shown as an enlarged $\times$). We can see that the maximum loss is very sensitive to outliers and its classification boundary in Fig. \ref{fig:synthetic full}~(a2) and Fig. \ref{fig:synthetic full}~(c2) are largely influenced by this outlier. Seen from Fig. \ref{fig:synthetic full}~(b2) and Fig. \ref{fig:synthetic full}~(d2), \atk loss is more robust with larger $k$ and achieves better classification results when $k \ge 3$.

%Case  3
\textbf{Case 3.} In the third row, there is no outliers and the $+$ samples and $-$ samples are still linear separable. However, the $+$ samples clearly has two distributions (typical distribution and rare distribution). Seen from Fig. \ref{fig:synthetic full}~(a3) and Fig. \ref{fig:synthetic full}~(c3), the linear classifiers learned from average loss sacrifice some $+$ samples from rare distribution even though the data are separable. This is because that the individual logistic loss has non-zero penalty for correctly classified samples and individual hinge loss has non-zero penalty for correctly classified samples with margin less than 1.  Hence samples that are ``too close'' to the classification boundary (samples from rare distributions in this example) are sacrificed to accommodate reducing the average loss over the whole datasets. Besides, average with hinge loss achieves better results than that with logistic loss, this may because that for correctly classified samples with margin larger than 1, the penalty caused by hinge loss is zero while that caused by logistic loss is still non-zeros. Hence this part of samples still has ``negative'' effect to the learned classification boundary of average with logistic loss. 
%By ``shifting down'' and truncating, \atk loss with proper $k$ can reduce the loss introduced by correctly classified samples with ``large'' margins and put more efforts to the hard samples that are misclassified or ``close'' to the classification boundary. 
By ``shifting down'' and truncating, \atk loss with proper $k$ (\eg, $k \in [1,18]$ for logistic loss and $k \in [1,50]$ for hinge loss) can better fit this data, as is shown in Fig. \ref{fig:synthetic full}~(b3) and Fig. \ref{fig:synthetic full}~(d3).
%For example, \atk loss with $k \in [1,18]$ for logistic loss and $k \in [1,50]$ for hinge loss can achieve better classification results, as is shown in Fig. \ref{fig:synthetic full}~(b3) and Fig. \ref{fig:synthetic full}~(d3). 

%Case  4
\textbf{Case 4.} The plots in the fourth row refers to a more complicated situation where there are both multi-modal distributions and outliers. Obviously, neither maximum loss (due to the outlier) nor average loss (due to the multi-modal distributions) can fit this data very well. Seen from Fig. \ref{fig:synthetic full}~(b4) and Fig. \ref{fig:synthetic full}~(d4), there exists a proper region of $k$ (\ie, $k \in [4,24]$ for logistic loss and $k \in [3,62]$ for hinge loss) that can yield much better classification results. We also report the linear classifier learned from AT$_{k=10}$ for better understanding. Seen from Fig. \ref{fig:synthetic full}~(a4) and Fig. \ref{fig:synthetic full}~(c4), the classification boundary of AT$_{k=10}$ is closer to the optimal Bayes linear classifier than that of maximum and average. 

%Case  5
\textbf{Case 5.} The fifth row shows an imbalance scenario where the $-$ samples are far less that the $+$ ones. The $+$ samples and $-$ samples are linear separable. We can see from Fig. \ref{fig:synthetic full}~(a5) that the average loss with individual logistic loss sacrifices all $-$ samples to obtain a small loss over the whole dataset. While the average loss with individual hinge loss obtains better results, it still sacrifices half of the $-$ samples, as is shown in Fig. \ref{fig:synthetic full}~(c5). In contrast, \atk loss can better fit this distributions and achieves better classification results with $k \in [1,25]$ for logistic loss and $k \in [1,135]$ for hinge loss.

%Case  6
\textbf{Case 6.} The sixth row shows an imbalanced data with one outlier. Comparing to the results in the fifth row, the performance of maximum loss decreases due to the outlier. The performance of average loss with hinge loss also decreases. Seen from Fig. \ref{fig:synthetic full}~(b6) and Fig. \ref{fig:synthetic full}~(d6), \atk loss with $k \in [2,12]$ for logistic loss and $k \in [3,59]$ for hinge loss can better fit this data and achieve better classification results.

Though very simple, these synthetic datasets reveal some properties of the maximum loss and average loss intuitively. That is, while maximum loss performs very well for separable data, it it very sensitive to outliers. Meanwhile, average loss is more robust to outliers than maximum loss, however, it may sacrifices some correctly classified samples that are ``too close'' to the classification boundary{, especially in imbalanced or multi-modal data distributions}. As the distributions of datasets from real applications can be very complicated and outliers are unavoidable, it is interesting and helpful to add an extra freedom $k$ to better fitting different data distributions.

\begin{table}
	\centering
	\footnotesize
	\renewcommand\arraystretch{1.2}
	\setlength\tabcolsep{5pt}
	
	\begin{tabular}{c c c c c | c c c c c || c c c}
		\hline
		\multicolumn{10}{c ||}{Binary Classification} & \multicolumn{3}{c}{Regression} \\
		\hline
		Dataset & $c$ & $n$ & $d$ & $IR$ & Dataset & $c$ & $n$ & $d$ & $IR$ & Dataset & $n$ & $d$\\
		\hline
		{\tt Monk} & 2 & 432 & 6 & 1.12 & {\tt Spambase} & 2 & 4601 & 57 & 1.54 & {\tt Sinc} & 1000 & 10 \\
		{\tt Australian} & 2 & 690 & 14 & 1.25 & {\tt German} & 2 & 1000 & 24 & 2.33 & {\tt Housing} & 506 & 13 \\
		{\tt Madelon} & 2 & 2600 & 500 & 1.0 & {\tt Titanic} & 2 & 2201 & 3 & 2.10 & {\tt Abalone} & 4177 & 8 \\
		{\tt Splice} & 2 & 3175 & 60 & 1.08 & {\tt Phoneme} & 2 & 5404 & 5 & 2.41 &  {\tt Cpusmall} & 8192 & 12 \\
		\hline
	\end{tabular}
	\caption{\small \em Statistical information of each dataset, where $c,n,d$ are the number of classes, samples and features, respectively. {\em IR} is the class ratio.}
	\label{dataset}
	~\vspace{-3.5em}
\end{table}

{\footnotesize
\begin{figure}
	%\begin{center}
	\begin{tabular}{c@{\hspace{0em}}c@{\hspace{0.5em}}|@{\hspace{0.5em}}c@{\hspace{0em}}c}
		%Case1
		\includegraphics[width=.24\textwidth]{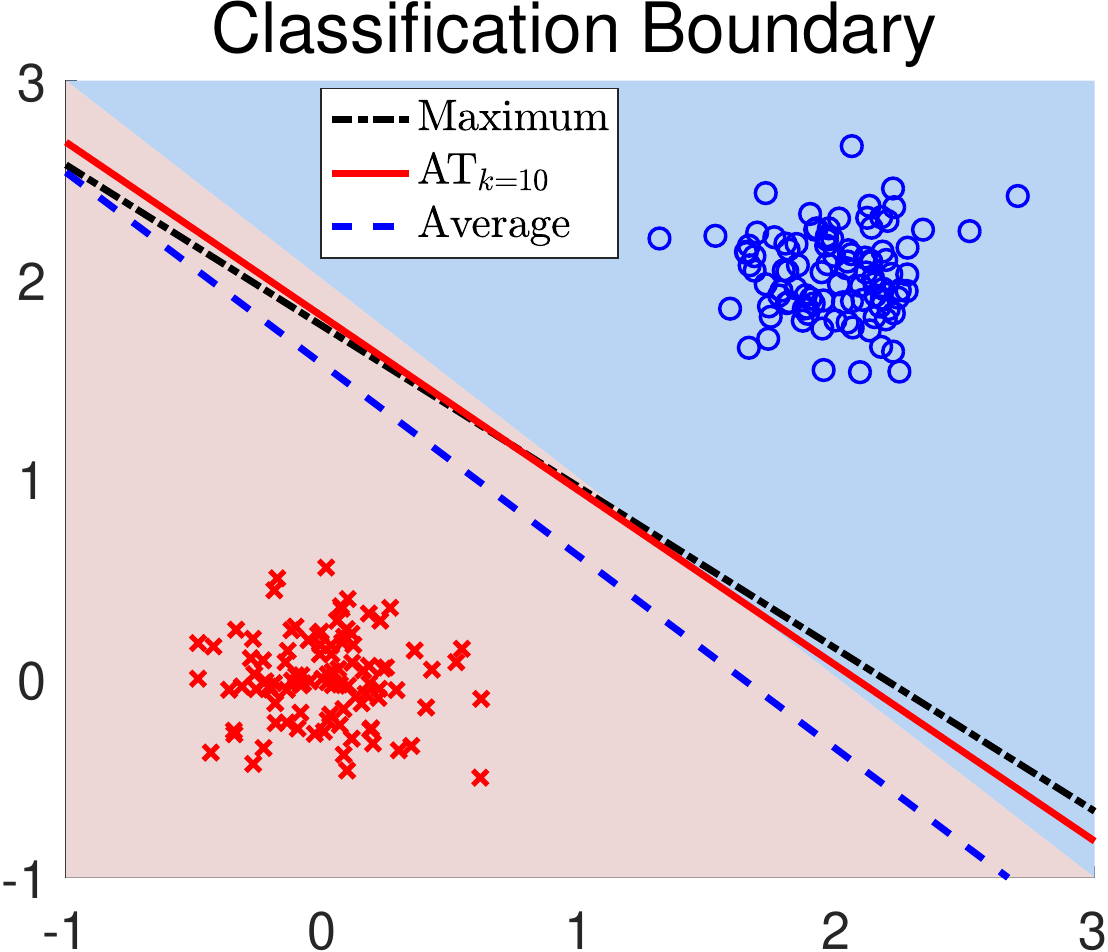} &
		\includegraphics[width=.24\textwidth]{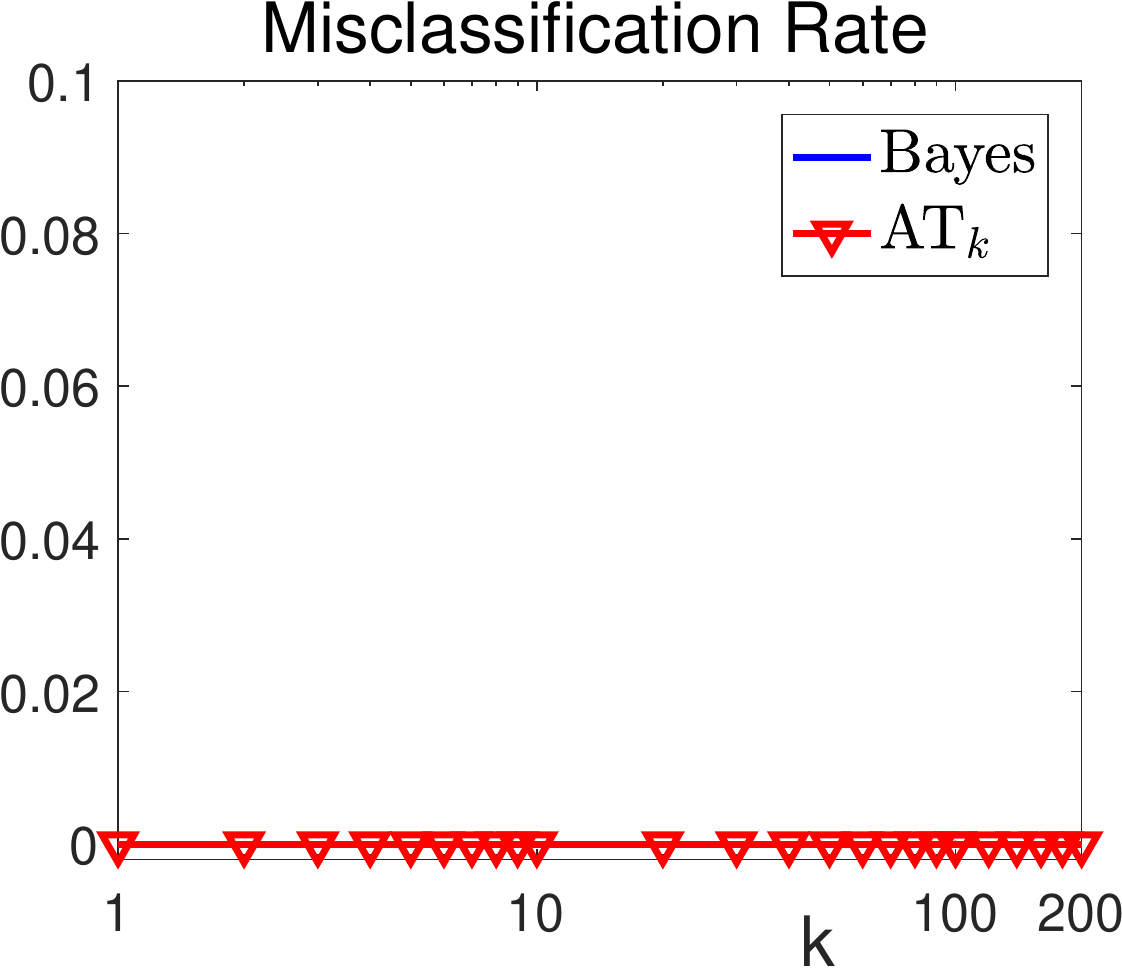}  &
		\includegraphics[width=.24\textwidth]{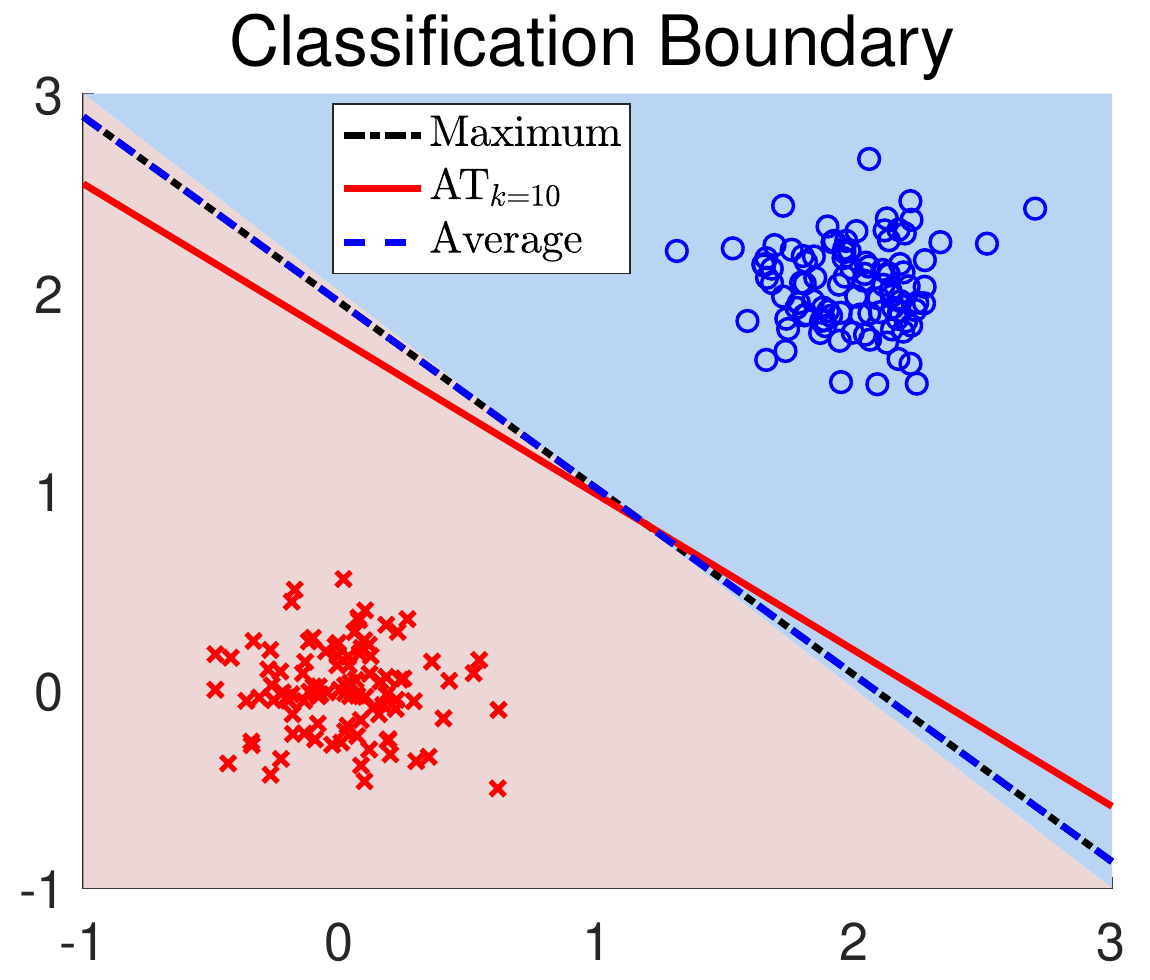} &
		\includegraphics[width=.24\textwidth]{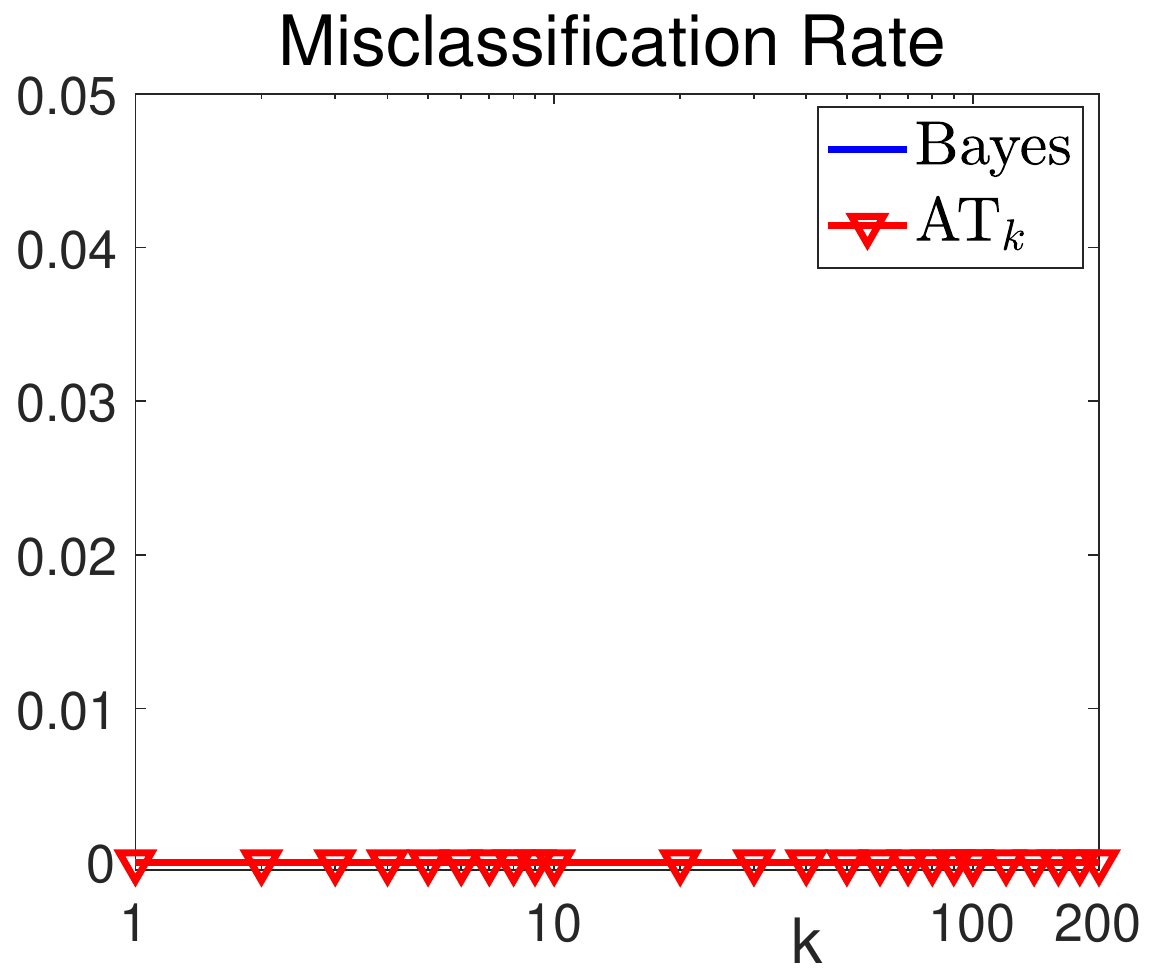} \\
		(a1) & (b1) & (c1) & (d1) \\
		%Case2
		\includegraphics[width=.24\textwidth]{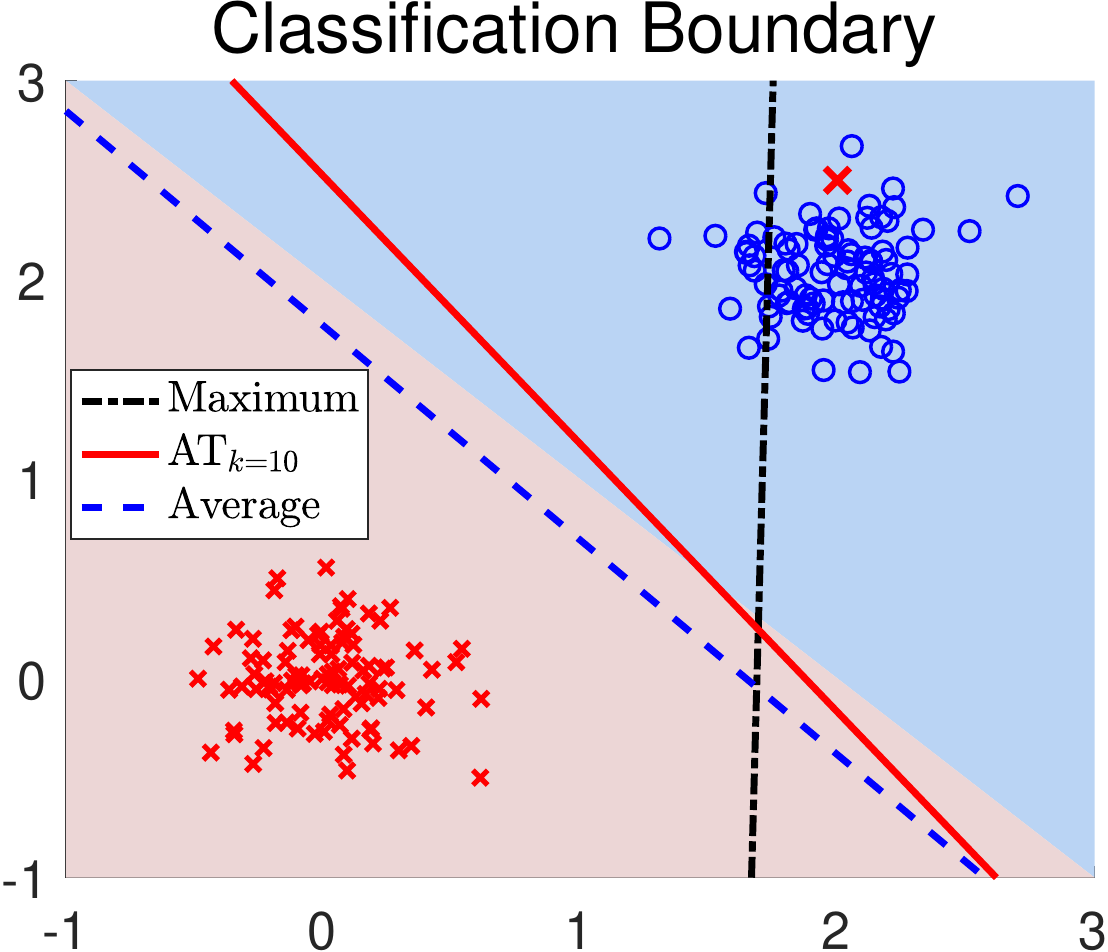} &
		\includegraphics[width=.24\textwidth]{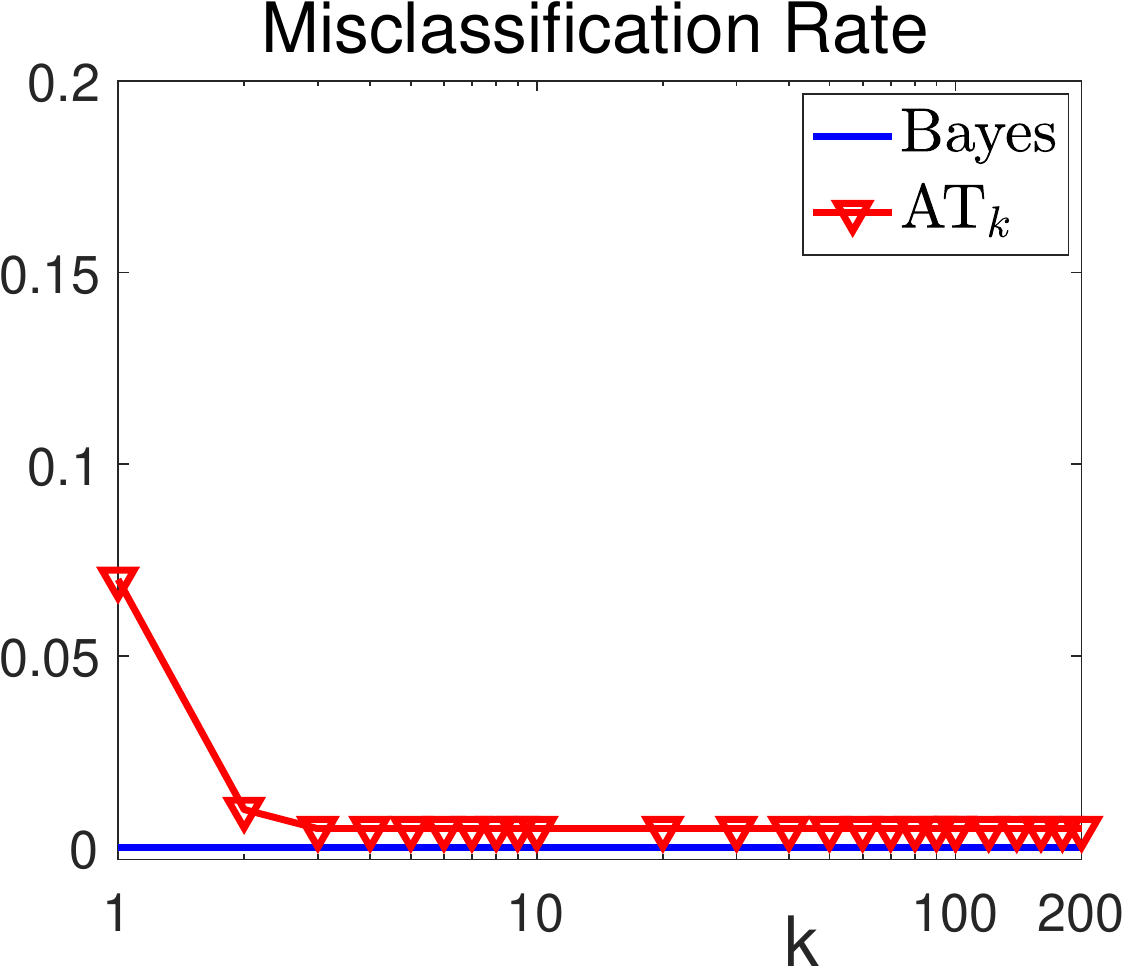}  &
		\includegraphics[width=.24\textwidth]{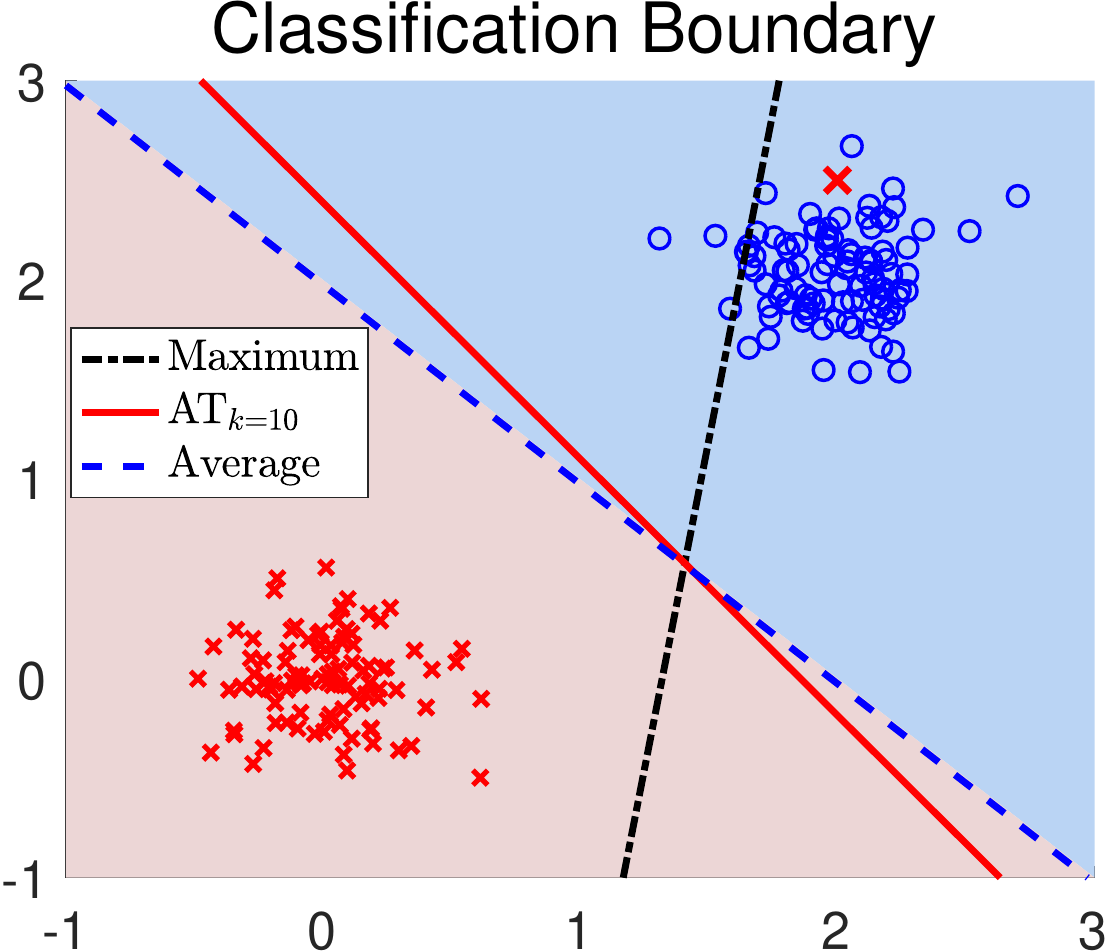} &
		\includegraphics[width=.24\textwidth]{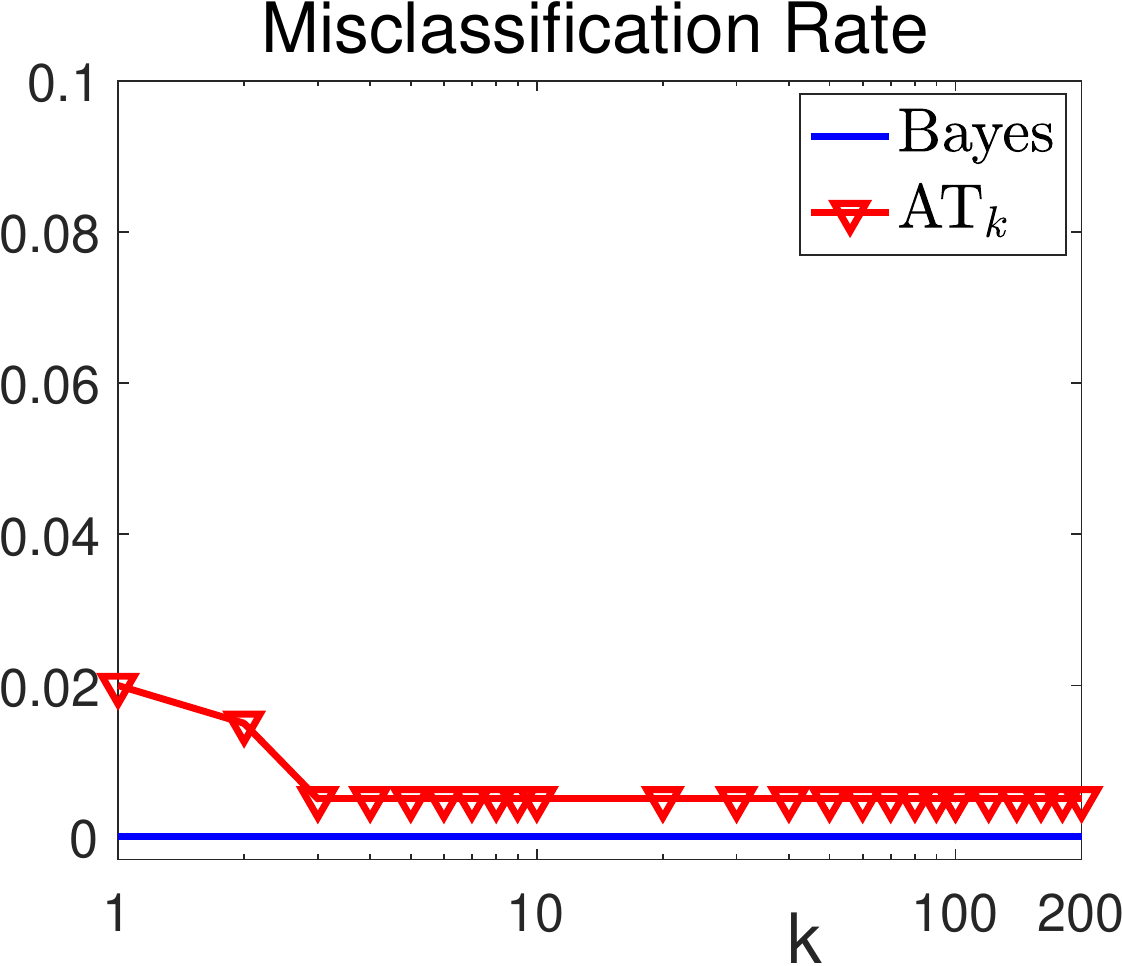} \\
		(a2) & (b2) & (c2) & (d2) \\
		%Case3
		\includegraphics[width=.24\textwidth]{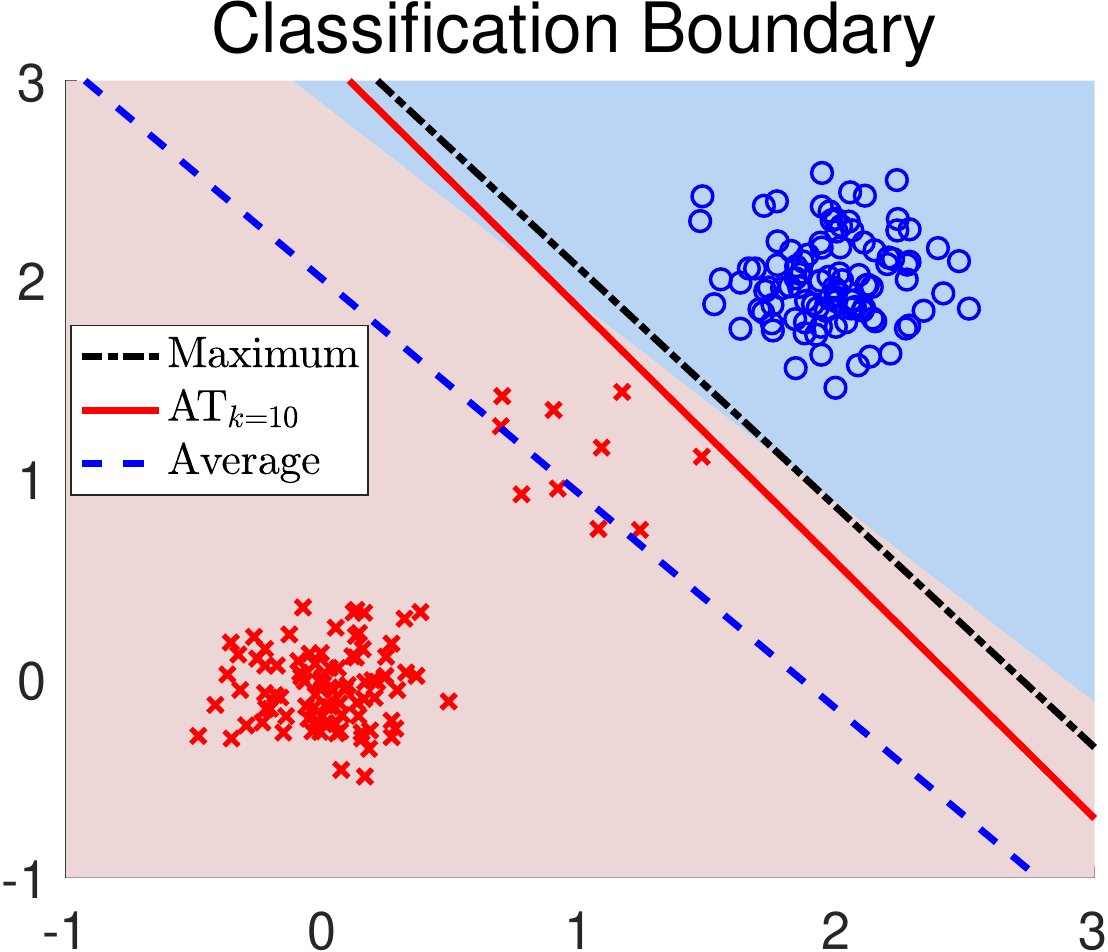} &
		\includegraphics[width=.24\textwidth]{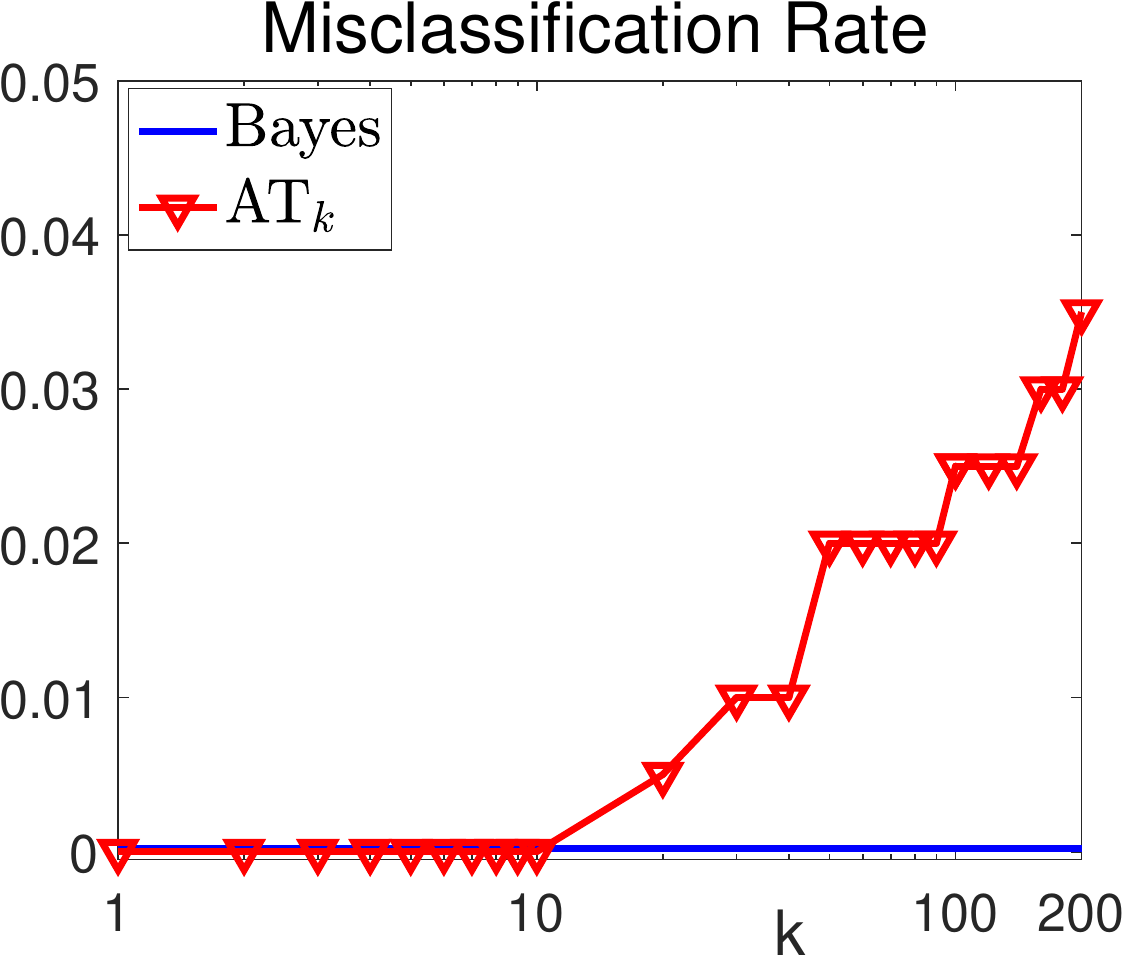}  &
		\includegraphics[width=.24\textwidth]{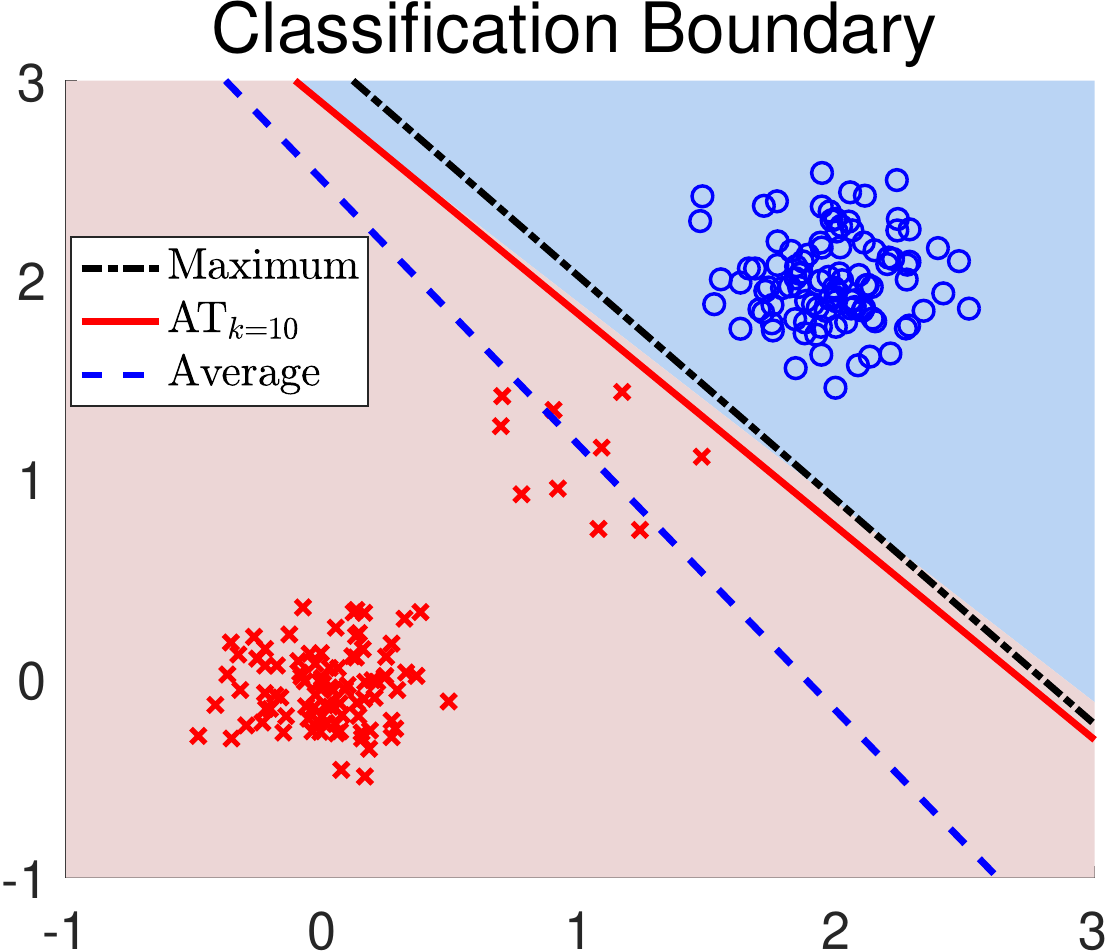} &
		\includegraphics[width=.24\textwidth]{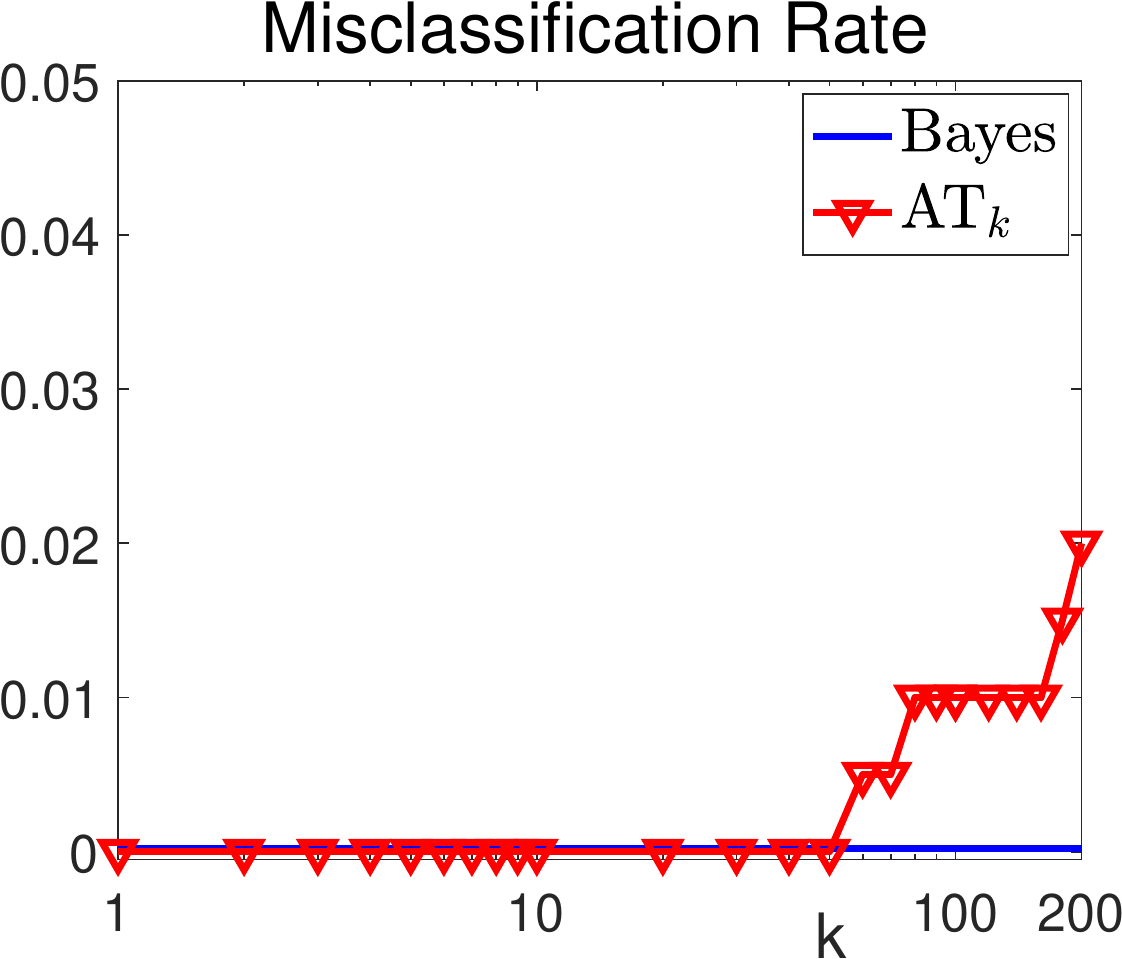} \\
		(a3) & (b3) & (c3) & (d3) \\
		%Case4
		\includegraphics[width=.24\textwidth]{logistic_classification_boundary_case4.pdf} &
		\includegraphics[width=.24\textwidth]{logistic_classification_error_case4.pdf}  &
		\includegraphics[width=.24\textwidth]{hinge_classification_boundary_case4.pdf} &
		\includegraphics[width=.24\textwidth]{hinge_classification_error_case4.pdf} \\
		(a4) & (b4) & (c4) & (d4) \\
		%Case5
		\includegraphics[width=.24\textwidth]{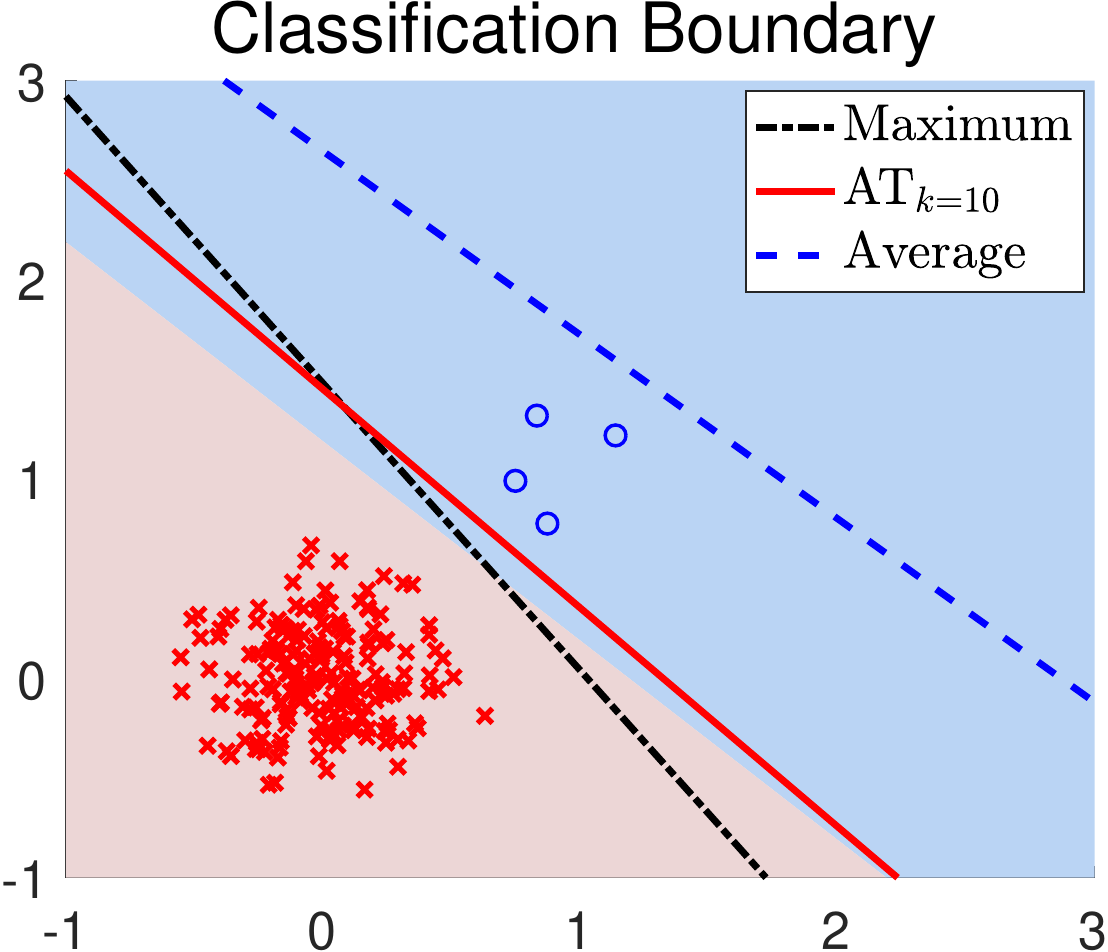} &
		\includegraphics[width=.24\textwidth]{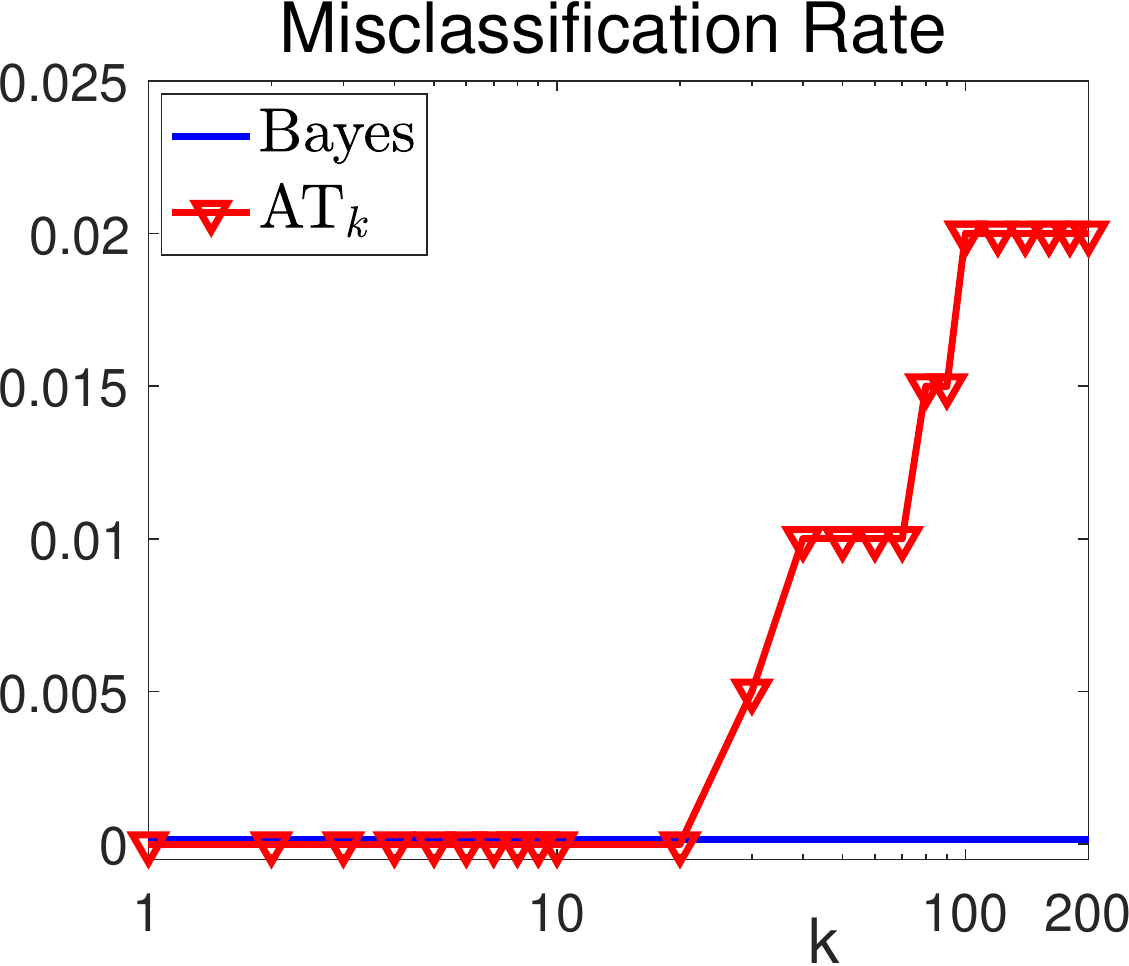}  &
		\includegraphics[width=.24\textwidth]{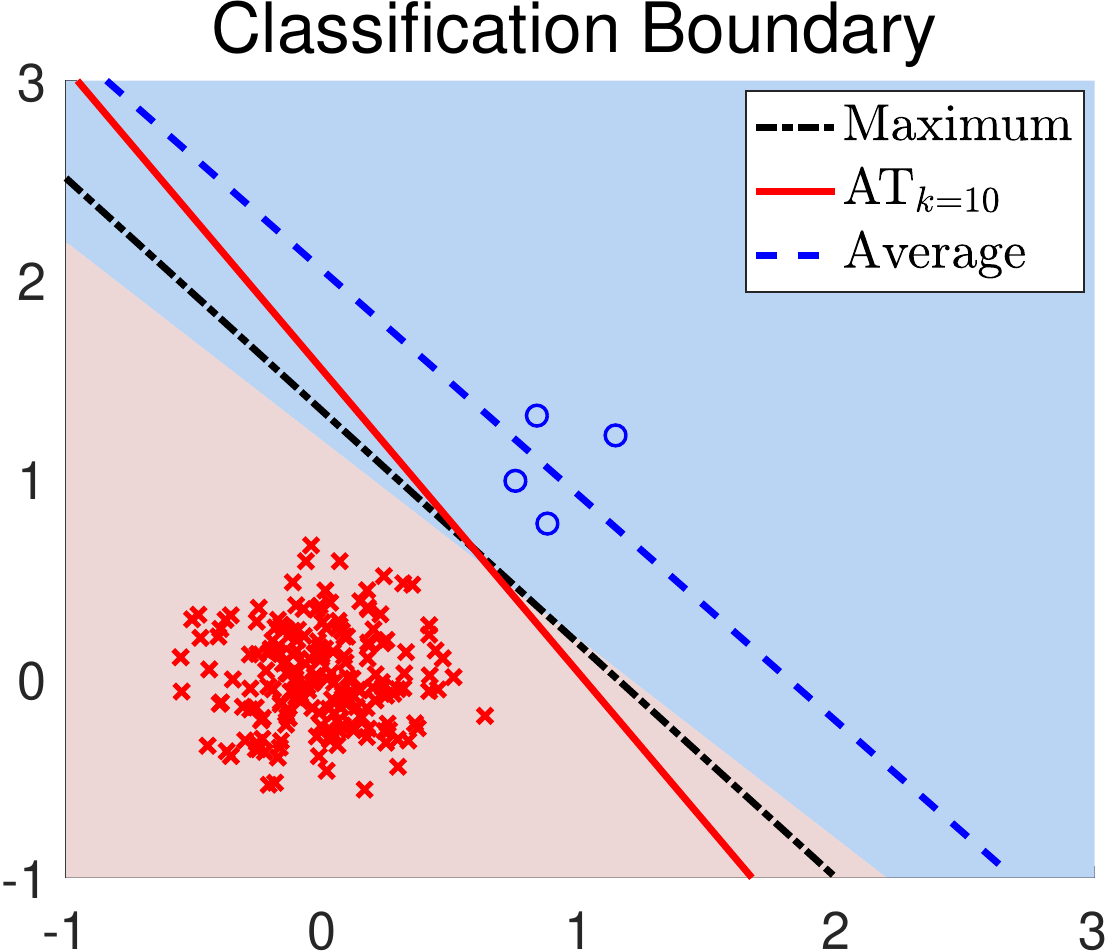} &
		\includegraphics[width=.24\textwidth]{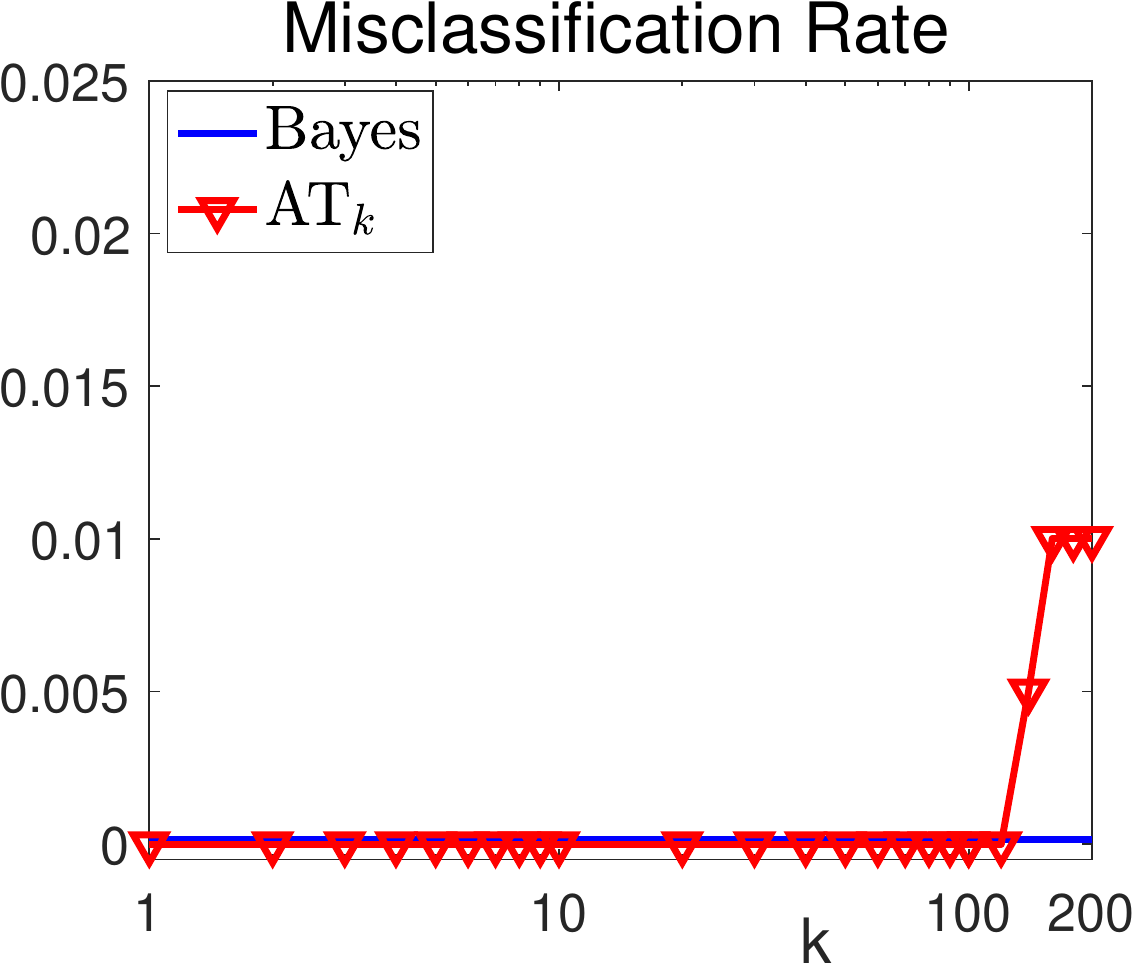} \\
		(a5) & (b5) & (c5) & (d5) \\
		%Case6
		\includegraphics[width=.24\textwidth]{logistic_classification_boundary_case6.pdf} &
		\includegraphics[width=.24\textwidth]{logistic_classification_error_case6.pdf}  &
		\includegraphics[width=.24\textwidth]{hinge_classification_boundary_case6.pdf} &
		\includegraphics[width=.24\textwidth]{hinge_classification_error_case6.pdf} \\
		(a6) & (b6) & (c6) & (d6) \\
	\end{tabular}
	%\end{center}
	~\vspace{-1em}
	\caption{\em \small Comparison of different aggregate losses on 2D synthetic datasets for binary classification on six different data distributions. Each row refers to one data distribution. In all plots, the $+$ samples are red crosses and the $-$ samples are blue circles. The outliers are shown with an enlarged $\times$ if any. The plots on the left panel report the results of linear classifiers learned with different aggregate losses combined with individual logistic loss, and that on the right panel are the results of different aggregate losses combined with individual hinge loss. The plots on the first and third columns show the learned linear classifiers of maximum, average and AT$_{k=10}$ with the optimal Bayes classification shown as shaded areas, and the plots on the second and fourth columns show the misclassification rate of \atk vs. $k$.}
	\label{fig:synthetic full}
	~\vspace{-2.5em}
\end{figure}
}

\textbf{Sinc data used for regression:} This dataset is drawn from sinc function, \ie, $y = \sin(x)/x$, where $x$ is an scalar, and the goal is to estimate $y$ from the input $x$. We randomly select $1000$ samples $(x_i, y_i)$ with $x_i$ drawn uniformly from $[-10,10]$. As we use linear regression model in our experiments, we map the input $x$ into a kernel space via the radial basis function (RBF) kernel. We select $10$ RBF kernels from $[-10,10]$, which leads to $10$-dimension input $\x = [k(x,c_1),\cdots,k(x,c_{10})]^T$, where $k(x,c_i) = \exp(-(x - c_i)^2)$. We also add random Gaussian noise $N(0,0.2^2)$ to the target output.

Table \ref{dataset} tabulates the statistical information of datasets that used in this paper. Experiments results in terms of G-mean for binary classification are reported in Table \ref{results:binary classification F1}, and experiments results in terms of \emph{mean absolute error (MAE)} for regression are also reported in Table \ref{results:regression mae}.

\begin{table}
	\centering
	\small
	\renewcommand\arraystretch{1.1}

	\begin{tabular}{c| c c c||c c c}
		%binary classificaiton, logistic loss
		\hline 
		\multirow{2}{*}{} & \multicolumn{3}{c||}{Logistic Loss} & \multicolumn{3}{c}{Hinge Loss}
		\\ \cline{2-7}
		& Maximum & Average & AT$_{k^\ast}$ & Maximum & Average & AT$_{k^\ast}$ \\
		\hline	
		
		{\tt Monk}   
		& 75.80(3.37)	 & 79.47(2.05)	 & \textbf{82.95(2.39)} 	 
		& 76.37(3.51)	 & 81.15(3.11)	 & \underline{82.68(2.79)} 	 \\ 
		
		{\tt Australian}
		& 78.88(7.56)	 & 86.10(3.19)	 & \textbf{88.37(2.97)} 	 
		& 78.99(7.47)	 & 85.72(3.15)	 & \underline{87.50(4.14)} 	 \\ 
		
		{\tt Madelon}
		& 51.20(2.92)	 & 59.28(1.41)	 & \textbf{60.26(1.58)} 	 
		& 49.42(2.71)	 & 59.36(1.83)	 & 59.72(1.51) 	 \\ 
		
		{\tt Splice}
		& 76.31(1.94)	 & 82.73(1.01)	 & \textbf{83.90(0.97)} 	 
		& 76.47(2.12)	 & \underline{83.78(1.12)}	 & \underline{83.79(0.97)} 	 \\ 
		
		{\tt Spambase}
		& 69.48(5.94)	 & 90.63(1.21)	 & 90.63(1.21) 	 
		& 69.96(6.76)	 & \textbf{91.90(0.85)}	 & \textbf{91.90(0.85)} 	 \\ 
		
		%imbalance
		{\tt German}
		& 44.88(7.37)	 & 60.12(7.59)	 & \textbf{63.80(4.29)} 	 
		& 44.87(7.34)	 & 61.02(7.49)	 & \underline{62.96(3.33)} 	 \\ 
		
		{\tt Titanic}
		& 46.52(15.27)	 & 66.69(1.44)	 & 66.69(1.44) 	 
		& 48.55(13.15)	 & 66.65(1.43)	 & \textbf{67.74(1.78)} 	 \\ 
		
		{\tt Phoneme}   
		& 19.10(11.84)	 & 63.00(1.84)	 & 66.29(2.04) 	 
		& 12.89(11.47)	 & \textbf{70.41(1.65)}	 & \textbf{70.41(1.65)} 	 \\ 
		
		\hline
	\end{tabular}
	\caption{\small \em Average G-mean(\%) of different learning objectives over $8$ datasets. The best results are shown in bold with results that are not significant different to the best results underlined.}
	\label{results:binary classification F1}
\end{table}

\begin{table}[t]
	\centering
	\footnotesize
	\renewcommand\arraystretch{1.1}
	\setlength\tabcolsep{1.5pt}
	
	\begin{tabular}{c| c c c||c c c}
		%binary classificaiton, logistic loss
		\hline 
		\multirow{2}{*}{} & \multicolumn{3}{c||}{Square Loss} & \multicolumn{3}{c}{Absolute Loss}
		\\ \cline{2-7}
		& Maximum & Average & AT$_{k^\ast}$ & Maximum & Average & AT$_{k^\ast}$ \\		
		\hline
		{\tt Sinc}   
		& 0.2438(0.0445) & 0.0816(0.0045) & \textbf{0.0806}(0.0044)
		& 0.1489(0.0466) & {0.0827}(0.0048) & {0.0821}(0.0055) \\
		%		\hline
		{\tt Housing}   
		& 0.1198(0.0150) & 0.0738(0.0075) & 0.0736(0.0079)
		& 0.1233(0.0127) & \underline{0.0713}(0.0089) & \textbf{0.0712}(0.0088) 	  \\
		
		%		\hline
		{\tt Abalone}   
		& 0.1312(0.0919) & 0.0575(0.0016) & 0.0574(0.0015) 
		& 0.1082(0.0303) & \underline{0.0559}(0.0014) & \textbf{0.0557}(0.0016)  \\
		
		{\tt Cpusmall}   
		& 0.2404(0.0832) & 0.0634(0.0027) & 0.0627(0.0025) 
		& 0.1868(0.0997) & \underline{0.0423}(0.0018) & \textbf{0.0422}(0.0018)  \\
		
		\hline
	\end{tabular}
%	\vspace{0.55pt}
	\caption{\small \em Average MAE on four datasets. The best results are shown in bold with results that are not significant different to the best results underlined.}
	\label{results:regression mae}
	~\vspace{-3.5em}
\end{table}
	
	{\small
		\bibliographystyleapx{plain}
		\bibliographyapx{refs_apx}
	}
	
\end{document}